\documentclass{article}

\usepackage{authblk}
\usepackage{stmaryrd}
\usepackage{algpseudocode}
\usepackage{algorithm}
\usepackage{tikz}
\usetikzlibrary{positioning,shapes,arrows}
\usepackage{tikz-cd}
\usepackage{subcaption}
\usepackage{yhmath}
\usepackage{array}
\usepackage{natbib}
\usepackage{hyperref}
\usepackage{url}
\usepackage{mathtools}
\usepackage{amssymb,amsthm}
\usepackage{fancyvrb}
\usepackage{scrextend}
\usepackage[hmargin=1.5in,vmargin=1.2in] {geometry} 

\usepackage[titletoc,title]{appendix}

\def \version { version 2 }
\title{%
Probabilistic Inference Using Generators  \\[16pt]
  \Large The \textit{Statues} Algorithm   \\[16pt] }

\author{Pierre Denis\thanks{independent scholar, Louvain-la-Neuve, Belgium -- pie.denis@skynet.be}}

\date{ \small \version -- \today }

\theoremstyle{definition}
\newtheorem{defn}{Definition}[section]

\newtheorem{thm}{Theorem}[section]
\newtheorem{prop}[thm]{Proposition}
\newtheorem{case}{Case}
\makeatletter\@addtoreset{case}{thm}\makeatother

\algnewcommand\algorithmicswitch{\textbf{switch}}
\algnewcommand\algorithmiccase{\textbf{case}}
\algnewcommand\algorithmicassert{\texttt{assert}}
\algnewcommand\Assert[1]{\State \algorithmicassert(#1)}%
\newcommand {\re} [1] { \Diamond #1 }
\newcommand {\rep} [1] { \re\,{\p{#1}} }
\newcommand {\bp} [1] {\widehat{#1}}
\newcommand {\tuple} [1] {[\,#1\,]}
\newcommand {\mixt} [1] {\rhd \bigl\{ #1 \bigl\}}
\newcommand {\dict} [1] {\bigl\{ #1 \bigl\}}
\newcommand {\utablere} [2] {#1 \unrhd #2}
\newcommand {\tablere} [2] {\utablere{#1}{\dict{#2}}}
\newcommand {\eq} [2] { #1 \hspace{-2.4pt}=\hspace{-2pt} #2 }
\newcommand {\efunc} [2] { \wideparen{#1}(#2) }
\newcommand {\egiven} [2] { #1 \obar #2 }

\newcommand {\etable} [2] { #1 \circledcirc \big\{ #2 \big\}}
\newcommand {\true} {\mathsf{true}}
\newcommand {\false} {\mathsf{false}}
\newcommand{\Lea} {Lea}

\newcommand{\mLea} {$\mu$Lea\;}
\newcommand {\dpd} [1] {\langle #1 \rangle}
\newcommand \given {\:\big\vert\:}
\newcommand {\dpdf} [1] {\bigl \{ \: #1 \: \bigl \}}
\newcommand {\pmf} [1] {\llbracket #1 \rrbracket}

\newcommand {\p} [1] { \mathsf{t}(#1) }
\newcommand {\Or} [3] {\bigvee\limits_{\substack{#1\\#2}} #3}
\newcommand {\PR} [1] { \Pr\big( #1 \big)}

\algdef{SE}[SWITCH]{Switch}{EndSwitch}[1]{\algorithmicswitch\ #1\ \algorithmicdo}{\algorithmicend\ \algorithmicswitch}%
\algdef{SE}[CASE]{Case}{EndCase}[1]{\algorithmiccase\ #1}{\algorithmicend\ \algorithmiccase}%
\algtext*{EndSwitch}%
\algtext*{EndCase}%

\algrenewcommand\algorithmicprocedure{\textbf{generator}}
\algnewcommand\Yield{\State \textbf{yield} }

\begin{document}
    
\maketitle

\begin{abstract}
We present here a new probabilistic inference algorithm that gives exact results in the domain of discrete probability distributions. This algorithm, named the Statues algorithm, calculates the marginal probability distribution on probabilistic models defined as direct acyclic graphs. These models are made up of well-defined primitives that allow to express, in particular, joint probability distributions, Bayesian networks, discrete Markov chains, conditioning and probabilistic arithmetic. The Statues algorithm relies on a variable binding mechanism based on the generator construct, a special form of coroutine; being related to the enumeration algorithm, this new algorithm brings important improvements in terms of efficiency, which makes it valuable in regard to other exact marginalization algorithms.

After introduction of several definitions, primitives and compositional rules, we present in details the Statues algorithm. Then, we briefly discuss the interest of this algorithm compared to others and we present possible extensions. Finally, we introduce Lea and MicroLea, two Python libraries implementing the Statues algorithm, along with several use cases.  A proof of the correctness of the algorithm is provided in appendix.
\end{abstract}


\section{Introduction}

Probabilistic inference is a field that is receiving renewed attention with latest developments in AI and probabilistic programming. Nowadays, problems characterized by some uncertainty can be modeled using different approaches, formalisms, primitives and levels of expressiveness: joint probability distributions, Bayesian networks, Markov chains, hidden Markov models, probabilistic arithmetic, probabilistic logic, etc. Many of these approaches can be further refined according to the type of random variables, whether discrete, continuous or mixed. The modeling approach depends, at least, on the kind of problem at hand, the availability of data and the provision of efficient algorithms to perform the required probability calculations.

In order to perform actual problem resolution, each modeling approach has its own catalogue of algorithms, characterized by different merits and trade-offs -- see \cite{russell_norvig} or \cite{de_raedt} for an overview. Several algorithms produce exact results but may be limited practically by complexity barriers whilst other algorithms can deal with intractable problems by delivering approximate results. In the specific case of Bayesian networks (BN), exact algorithms include enumeration, belief-propagation, clique-tree propagation, variable elimination and clustering algorithms; on the other hand, approximate algorithms include rejection sampling, Gibbs sampling, Monte-Carlo Markov Chain (MCMC). About the computational complexity of probabilistic inference, it has been proved that the problem of inference in unconstrained BN is NP-hard \citep{cooper}. This statement has set some limitations on exact algorithms and probably gave momentum on the research for efficient approximate inference algorithms; however, those approximate algorithms were in turn acknowledged for similar complexity categorization \citep{dagum}. Beside the Bayesian reasoning domain, probabilistic arithmetic and, more generally, the study of deterministic functions applied on random variables ($+$, $-$,  $\times$, $/$, min, max, etc.) constitutes a research field on its own; we can cite at least convolution-based approaches \citep{williamson, williamson_downs,agrawal,evans} and discrete envelope determination (DEnv) \citep{berleant_1998,berleant_2003}.

These well-established algorithms, in their original formulations, are specialized for one single modeling approach. In particular, algorithms for probabilistic arithmetic do not handle Bayes reasoning or even simple conditioning; on the other hand, above-cited inference algorithms for BN do not handle arithmetic (e.g. the sum of two random variables of the BN, whether latent or observed). Also, many BN algorithms handle only \emph{observations} expressible as conjunctions of equalities; without extensions, these algorithms cannot treat the \emph{conditioning} in its generality, that is considering any boolean function of the BN variables as a possible assertion. In short, early probabilistic models and associated algorithms has been constrained by some compartmentalization. These limitations tend now to disappear with the advent of probabilistic programming (PP) and richer probabilistic models that can mix several approaches together. Creating new efficient algorithms --or improving existing ones-- to perform exact inference on these probabilistic models is incidentally an active domain of research.

Following this trend, the present paper introduces a new algorithm for the task of exact marginalization inference in the scope of discrete random variables. It is in essence a variant of the enumeration algorithm that provides important improvements for reducing both the number of paths to explore and the number of calculations to perform. The enabler of this algorithm is the \emph{generator} construct, a special case of coroutine \citep{knuth,saba}, which is available in several modern programming languages. The generators provide a great interest for combinatorial generation \citep{saba} even though these seem to be overlooked in computer science literature: generators are not widely used in published algorithms, for which the subroutine construct is prevalent. To the best of our knowledge, at the time of writing, no probabilistic inference algorithm using generators have been published yet.\footnote{The \emph{continuation} construct, originated from functional languages is another way to achieve coroutines. It is worth pointing out that "continuation passing style" (CPS) is used in marginalization algorithms of WebPPL, a modern probabilistic programming language based on JavaScript \citep{dippl}.}

The paper is organized as follows. Section \ref{sect:RV} introduces a probabilistic modeling framework using discrete random variables; this framework defines a set of primitives to build up probabilistic models; some examples provide substantiations that these primitives are rich enough to express (in particular) joint probability distributions, multivariate variables, conditioning, Bayesian networks, Markov chains and probabilistic arithmetic. Section \ref{sect:eval} sets up the problem of marginalization inference as calculating the probability distribution of a derived random variable; the concept of \emph{p}-expression is defined as a general data structure to represent dependencies between random variables as a direct acyclic graph (DAG); then, the Statues algorithm is presented in details. Section \ref{sect:discussion} discusses the salient points of the Statues algorithm compared to some well-established algorithms. Section \ref{sect:extensions} describes possible extensions. Section \ref{sect:lea} introduces the Lea and MicroLea libraries, two implementations of the Statues algorithm, along with several use case snippets. Three appendixes are provided: appendix \ref{sect:generators} makes a short introduction to generators, appendix \ref{sect:implem_design} gives hints on implementation and appendix \ref{sect:proof} provides a proof of the correctness of the algorithm.

\section{Probabilistic modeling framework} \label{sect:RV}

We shall model randomness using discrete random variables with a finite domain. We do not put any restriction on the domains provided that these are discrete and finite: these can be numbers, matrices, symbols, booleans, tuples, functions, propositions, etc\footnote{The term "random variable" is \textit{stricto senso} specific to real number domains. This limitation is deliberately set aside here for the sake of generality. To be rigorous, we should use the term "random element", which comes from the pioneering work of \citet{frechet} and which subsumes the definition of a random variable. Also, we deliberately avoid the mathematical formalism of probability spaces ($\Omega, \mathcal{F},\mathcal{P})$ even if the present framework could be expressed using this formalism.}. Also, it is \emph{not} required to have an order relationship on the domain; such requirement is compulsory to calculate the cumulative distribution function but such function is not needed in our approach. In the context of the present paper, we shall abbreviate the object characterized above as a "random variable" or simply "RV".

In the following subsections, we shall introduce the different types of RVs, as primitives to build up probabilistic models. We shall distinguish \emph{elementary} RVs, which are defined on their own, and \emph{derived} RVs, which are defined in terms of other RVs.

\subsection{Elementary random variables}

\begin{defn}
An \emph{elementary random variable} is a random variable with a given finite domain and characterized by a given prior probability mass function.
\end{defn}

Elementary RVs are the most basic RVs. They require specifying prior probability for each possible value of their domains. Note that, since we constrain the domain of RV to be discrete and finite, an elementary RV can be called also a \emph{categorical distribution}. The probability mass function (pmf) shall obey Kolmogorov axioms: the individual probabilities shall be nonnegative and the sum of all probabilities over the domain shall be 1. We exclude the Poisson and hypergeometric distributions, which are discrete but not finite; such distributions could however be approximated, for example, by considering only the finite set of values having a probability above a given threshold and normalizing the probabilities to have a total of 1. Continuous random variables are excluded but their probability density functions can be approximated through discretization; several methods exist for this purpose, with known shortcomings \citep{berleant_1998,agrawal}.

An example of elementary RV is the result obtained by flipping a fair coin. We can model this by defining RV $F$ with dom$(F) \triangleq \{ \mathsf{tail},\mathsf{head} \} $ and a uniform pmf defined using the following notation -- borrowed from \cite{williamson}:
$$ F \sim \dpdf{ (\mathsf{tail} , \tfrac{1}{2}),(\mathsf{head} , \tfrac{1}{2}) } $$

Another example is the value $D$ got after throwing a fair die:
$$ D \sim \dpdf{(\text{1},\tfrac{1}{6}),(\text{2},\tfrac{1}{6}),(\text{3},\tfrac{1}{6}),(\text{4},\tfrac{1}{6}), (\text{5},\tfrac{1}{6}), (\text{6},\tfrac{1}{6}) } $$

For the sake of simplicity, we assume here that all pmf are defined by extension (as above), even if other formulation could be handled without much difficulties. Also, we shall forbid values with a null probability and duplicate values; this can be done easily by a \emph{condensation} treatment \citep{kaplan}, which removes elements with null probabilities and merges equal elements together while adding their probabilities. For instance the following non-condensed pmf $\dpdf{ (\mathsf{tail} , \tfrac{1}{2}),(\mathsf{head},\tfrac{3}{8}),(\mathsf{head},\tfrac{1}{8}),(\mathsf{tie},0) }$ is equivalent to the above-defined $F$ pmf.


It is important to avoid confusion between a RV and the pmf that characterizes it. In the following, a letter with a hat (e.g. $\bp{a}$) refers to a given pmf; by prepending a diamond on a given pmf, we designate an elementary RV characterized by this pmf. So, $\re{\bp{a}}$ is a RV having $\bp{a}$ as pmf or, for short,
$$ \re{\bp{a}} \; \sim \; \bp{a} $$
Thanks to this formalism, the elementary RV $F$ seen above could equivalently be defined as
$$ F \coloneqq \re{\dpdf{(\mathsf{tail} , \tfrac{1}{2}) , (\mathsf{head} , \tfrac{1}{2})}} $$

As a special case, we admit any elementary RV having a domain of one unique element; such RV is then certain and has a probability of 1. For instance, the usual $\pi$ number can be represented by the RV $\re{\dpdf{(\pi,1)}}$. Even if there is no randomness in such dummy RV, this assimilation shall simplify our inference algorithm.

Let us stress that any two distinct elementary random variables are independent by definition. Each occurrence of the diamond notation creates a brand new independent RV, even if applied on the same pmf. For example, let us consider the pmf $ c \coloneqq \dpdf{ (\mathsf{tail} , \tfrac{1}{4}),(\mathsf{head} , \tfrac{3}{4}) }$ and the two boolean RVs defined as $A \coloneqq \re{c}$ and $B \coloneqq \re{c}$. Then, $A$ and $B$ represent two independent events, e.g. two throws of the same biased coin. To emphasize this fact, we can compare the defined RVs together: the equality $A = A$ is always true but the equality $A = B$ has only a probability $\tfrac{5}{8}$ to be true since $\Pr(A=B) = \tfrac{1}{4}.\tfrac{1}{4} + \tfrac{3}{4}.\tfrac{3}{4}$. This topic will be elaborated later, by defining the concepts of referential consistency (\ref{sect:ref_const}) and functional RVs (\ref{sect:func_RV}).

Since we admit any domain for our elementary RVs, two special cases are worth mentioning: boolean RVs and joint probability distributions.
 
\subsubsection{Boolean random variables}

\begin{defn}
A RV $C$ is defined as \emph{boolean} iff the domain of $C$ contains no other values than booleans $\{\true,\false\}$.
\end{defn}

For convenience, we shall use the notation $\p{p}$ to represent a pmf with a given probability $p$ to be true:
$$ \p{p} \triangleq \dpdf{(\true , p) ,\: (\false , 1-p)} $$

For instance, the elementary RV defining that a fair die shows the value 4 can be notated $\rep{\tfrac{1}{6}}$. Note that, as trivial special cases, we can write the  relations $\true \sim \p{1}$ and $\false \sim \p{0}$. The classical $\Pr(A)$ notation representing the probability of occurrence of a given event $A$ can then be defined as the inverse of the above-defined notation, i.e.
$$ \Pr(\rep{p}) \triangleq p $$

\subsubsection{Joint probability distributions} \label{sect:jpd}
The concept of elementary RV allows defining joint probability distributions (also known as multivariate distributions). The way to proceed consists in defining a RV with a set of tuples for domain; each tuple represents a possible outcome and each element of the tuple represents a given attribute (or measure) of this outcome. For example, here is a joint probability distribution linking the weather and someone's mood:
\begin{align*}
J \coloneqq \re{ \dpdf{ & (\tuple{\mathsf{rainy},\mathsf{sad}}, 0.20), \\
                        & (\tuple{\mathsf{rainy},\mathsf{happy}}, 0.10), \\
                        & (\tuple{\mathsf{sunny},\mathsf{sad}}, 0.05), \\
                        & (\tuple{\mathsf{sunny},\mathsf{happy}}, 0.65)} }
\end{align*}

Joint probability distributions allow modeling interdependence between random phenomena. This is the case in the example above since, in particular, the joint probability $\Pr(J=\tuple{\mathsf{sunny},\mathsf{happy}}) = 0.65 $ is not equal to the product of marginal probabilities $\Pr(J_1=\mathsf{sunny}).\Pr(J_2=\mathsf{happy}) = (0.05+0.65).(0.1+0.65) = 0.525$. Such construction is of course not suited when the number of outcomes or attribute become large; in section \ref{sect:table_RV}, we shall see how to model Bayesian networks, which allows trading off generality with compactness.

\subsection{Derived random variables} \label{sect:derived_RV}

Beside elementary RV, a random variable may also be defined in terms of other random variables.

\begin{defn}
A \emph{derived random variable} is a random variable that is defined by a given deterministic dependency on a given finite set of random variables.
\end{defn}

The basic idea of this recursive definition is that elementary RV can be used to define derived RV that, in turn, can be used to define other derived RV and so on up to elementary RVs. Of course, the terms "deterministic dependency" used in the definition is vague; the precise definition of these terms shall be elaborated throughout the present section. Note that, contrarily to elementary RV, the pmf of a derived RV is not given a priori, it shall result from a calculation. The sole requirement that we will put on the dependency is that it must allow calculating the exact pmf in a finite time, from the known pmf of the underlying elementary RVs.

We shall define the dependency of a derived RV as belonging to one of the four types \emph{tuple}, \emph{functional}, \emph{conditional} and \emph{table}. These are detailed in the following subsections, after introducing the concept of referential consistency.

\subsubsection{Referential consistency} \label{sect:ref_const}

Before covering the types of derived RVs, it is important to state a general rule that they shall obey: when defining a derived RV $Y$ depending on a given RV $X$, the value randomly chosen for each occurrence of $X$ shall be the same. To give a simple example, consider an elementary RV $X$ having $\{0,1\}$ as domain; if we define the RV $Y \coloneqq X+X$, then occurrence of $Y$ and the two occurrences of $X$ refer to the same outcome, which is the unique drawing of 0 or 1; therefore, the domain of $Y$ is $\{0,2\}$. We shall refer to this constraint as \emph{referential consistency}\footnote{Actually, this statement would have been automatically granted if we had used definitions in terms of probability spaces, since the random variables refer to the same outcome $\omega$ of the sample space $\Omega$. We judge that it is better to make this statement explicit anyway because, as we will see, it puts important constraints on the representation of the probabilistic model and on the inference algorithm that treats it.}. Many examples showing the importance of the concept will be given throughout next sections.

Note that this constraint is closely linked to the concept of \emph{stochastic memoization} found at least in Church \citep{goodman} and WebPPL \citep{dippl}. The two concepts actually enforce the same consistency constraint. The difference lies, to the best of our knowledge, in the fact that referential consistency applies on \emph{exact} probabilistic inference, whereas stochastic memoization applies on \emph{approximate} probabilistic inference (e.g. MCMC).

\subsubsection{Tuple random variables} \label{sect:tuple_RV}
The first type of derived RV is defined by grouping a given set of RVs into one tuple.
\begin{defn}
Be $n$ random variables $X_1, \;\ldots\; , X_n$ with $n \ge 1$. The RV $T$ defined as the tuple
$$ T \coloneqq \tuple{X_1, \;\ldots\; , X_n} $$
is called a \emph{tuple RV}.
\end{defn}
For instance, we can define a 2-tuple RV $T$ that is made up of two elementary RV having Bernoulli distributions:
\begin{align*}
B_1 & \coloneqq \re{\dpdf{(0,\tfrac{1}{2}), \; (1,\tfrac{1}{2})}} \\
B_2 & \coloneqq \re{\dpdf{(0,\tfrac{3}{4}), \; (1,\tfrac{1}{4})}} \\
T & \coloneqq \tuple{B_1,B_2}
\end{align*}
The pmf of $T$ can be calculated by enumeration:
$$ T \sim \dpdf{(\tuple{0,0},\tfrac{3}{8}), \; (\tuple{0,1},\tfrac{1}{8}), \; (\tuple{1,0},\tfrac{3}{8}), \; (\tuple{1,1},\tfrac{1}{8})} $$

Note that a tuple is defined as a sequence of elements, so the order of these elements is significant. Let us note for instance that swapping the inner RV in the tuple definition results in another distribution:
$$ U \coloneqq \tuple{B_2,B_1} $$
$$ U \sim \dpdf{(\tuple{0,0},\tfrac{3}{8}), \; (\tuple{0,1},\tfrac{3}{8}), \; (\tuple{1,0},\tfrac{1}{8}), \; (\tuple{1,1},\tfrac{1}{8})} $$

A rather contrived example is given when a RV appears twice in the same tuple:
$$ V \coloneqq \tuple{B_2,B_2} $$
Then, the referential consistency forces the two elements to be the same; the pmf is then
$$ V \sim \dpdf{(\tuple{0,0},\tfrac{3}{4}), \;  (\tuple{1,1},\tfrac{1}{4})} $$

Note that a tuple RV containing elementary RV (as seen here) is not equivalent to an elementary RV containing tuples (as seen in \ref{sect:jpd}, for joint probability distributions). In both cases, the domain is a set of tuples; however, a tuple RV is a derived RV and it cannot be used to \emph{specify} a joint probability distribution. Note also that, by definition, the empty tuple $\tuple{}$, is \emph{not} a tuple RV: it is an elementary RV having a probability of 1.

For non-empty tuples, we shall adopt the following LISP-like notation
$ \tuple{H \centerdot T} $ to represent a tuple with $H$ as first element (the "head") and $T$ as a tuple with remaining elements (the "tail"). So, the tuple RV $\tuple{B_1,B_2}$ defined above could be written as $\tuple {B_1 \centerdot \tuple{ B_2 \centerdot \tuple{}}}$.

\subsubsection{Functional random variables} \label{sect:func_RV}
The second type of derived RV is defined by application of a function on other RVs.
\begin{defn}
Be a random variables $X$ and an unary function $f$ which domain includes the domain of $X$. The RV $Y$ defined as 
$$Y \coloneqq f(X)$$
is called a \emph{functional RV}.
\end{defn}
Let us stress that $f$ is meant here to be a \emph{pure} function, that is deterministic and without side-effect: once the value of argument RV is defined ($X = x$), the value of $Y$ is uniquely defined ($Y = f(x)$). A functional RV can use any algorithm, provided that it can evaluate the result in a finite time, whatever the value given in argument. $n$-ary functions with $n > 1$ can easily be converted to unary functions by packing the arguments into a tuple RV \footnote{For instance, a 2-ary function $g$ shall be converted into a unary function $g'$, such that $g'(\tuple{X,Y}) \triangleq g(X,Y)$.}. Such treatment may seem odd at this stage but we shall see later that it makes the inference algorithm simpler.

Functional RVs cover, among others, a large set of basic mathematical operations (we assume in the following that $N$, $X$, $Y$, $Z$ have numerical domains, and $A$, $B$ have boolean domains):

\begin{itemize}
\item arithmetic: $X+Y$, $X-Y$, $X.Y$, $-X$, $\sqrt{X}$, $X^Y$, etc 
\item comparison: $X = Y$, $X \neq Y$, $X < Y$, $X \le Y$, etc 
\item logical: $\overline{A}$, $A \wedge B$, $A \vee B$, $A \Rightarrow B$, $A \Leftrightarrow B$, etc 
\end{itemize}
and any combinations of these operations, like
$$F \coloneqq \left(N \ge 3 \right) \wedge \left(X^N + Y^N = Z^N \right)$$
which use standard function composition. Note that, to strictly conform to our definition, the value 3 is here considered as an elementary RV giving 3 with a probability 1; also, the infix subexpressions shall be translated using unary functions as explained above, viz.
$$F \coloneqq  and( \tuple{ge(\tuple{N,3}), eq( \tuple{ add(\tuple{pow(\tuple{X,N}),pow(\tuple{Y,N})}), pow(\tuple{Z,N})})}) $$

Let us come back on the rule of referential consistency (\ref{sect:ref_const}). To exemplify the idea, let us consider the following dummy functional RVs: $(X-X)$ is certainly $0$, $(X+X=2X)$ is certainly true, $(X + Y < Y + X)$ is certainly false, etc; also, assuming that the RV's domains are natural numbers, the RV $F$ defined above is certainly false, as stated by the last Fermat theorem proved by \citet{wiles}! To complete the topic, note that the referential consistency holds also if intermediate RVs are defined; for example, if we define $ S \coloneqq (X+Y)^2 $, $U \coloneqq X^2 + Y^2 $ and $ V \coloneqq S - U - 2XY $, then $V$ is certainly $0$. We can remark that the lack of referential consistency is referred with the terms "dependency error" in \citet{williamson} and \citet{williamson_downs}. In contrast to these authors who investigate how these errors can be bounded, we shall outlaw here any such dependency error. As we shall see, the referential consistency is essential in our approach: it enables, amongst others, conditioning and Bayesian inference.

Beside the afore-mentioned common mathematical functions, we could add many other useful functions: checking the membership of an element in a given set, taking the minimum/maximum element of a tuple, summing the elements of a tuple, getting the attribute of an object, etc. Among these functions, the indexing of a given tuple $t$ is worth to mention. As an illustration, let us define $extract(\tuple{t,i})$ as the function giving the $i^{th}$ element of $t$; reconsidering the joint probability distribution $J$ seen in section \ref{sect:jpd}, we see that weather and mood can be defined as functional RVs, respectively $J_1 \coloneqq extract(\tuple{J,1})$ and $J_2 \coloneqq extract(\tuple{J,2})$. Then, the calculations showing the interdependencies between these two RVs can be redone relying on referential consistent and noting that
\begin{align*}
\Pr(\eq{J_1}{\mathsf{sunny}} \wedge \eq{J_2}{\mathsf{happy}}) & \; = \; \Pr(\eq{J}{\tuple{\mathsf{sunny},\mathsf{happy}}}) \\
  & \; \neq \; \Pr(\eq{J_1}{\mathsf{sunny}}) \Pr(\eq{J_2}{\mathsf{happy}})
\end{align*}

\subsubsection{Conditional random variables} \label{sect:cond_RV}
The third type of derived RV is defined by filtering the values of one given RV according to a condition expressed in a given boolean RV.
\begin{defn}
Be the RV $X$ and the boolean RV $E$. The RV $C$ defined as $X$ under the condition $E$ is noted as
$$ C \coloneqq X \given E $$
and is called a \emph{conditional RV}.
\end{defn}

The idea is here to build a new RV from an existing one $X$, with the assurance that the possible values of $X$ are such that the given condition $E$ is true. $E$ could represent an evidence, an assumption or a constraint; $E$ has its own prior probability to be true but, in the present context, it is assumed that it \emph{is} certainly true. Note that it is wrong to consider that $X$ \emph{changes} when evidence $E$ is provided; actually, $X$ keeps its definition unchanged whatever it may happen; $C$ is just a new RV, which is meant to capture some evidence absent from $X$'s definition.

Although not required, the interesting cases happen of course when $X$ and $E$ are dependent each from each other; this occurs if the evidence is a functional RV referring to the conditioned variable  -- i.e. $ X \given h(\tuple{\ldots\;,X,\;\ldots}) $ -- or, more generally, if the evidence and conditioned variables are both functional RVs referring to the same RV -- e.g. $ f(\tuple{\ldots\;,Y,\;\ldots}) \given g(\tuple{\ldots,Y,\ldots})$. As a matter of fact, referential consistency on "shared" RVs like $Y$ is essential for the semantic of conditional RV.

For example, let us define $D_1$ and $D_2$ as the respective values of two fair dice and $D$ as the sum of these two values:
\begin{align*}
& \bp{d} \coloneqq \dpdf{(\text{1},\tfrac{1}{6}),(\text{2},\tfrac{1}{6}),(\text{3},\tfrac{1}{6}),(\text{4},\tfrac{1}{6}), (\text{5},\tfrac{1}{6}), (\text{6},\tfrac{1}{6}) } \\
& D_1  \coloneqq \re{\bp{d}} \\
& D_2  \coloneqq \re{\bp{d}} \\
& D  \coloneqq D_1 + D_2 
\end{align*}
and let us assume that we know, by any means, that first die shows 1 and that the dice total is greater than 5: the conditional RV for the dice total is then written
$$ D \given D_1 = 1 \wedge D > 5 $$ 
which is characterized by the pmf $ \dpdf{(\text{6},\tfrac{1}{2}),(\text{7},\tfrac{1}{2}) } $. To make a link with the classical concept of conditional probability, we simply need to transform the above-defined RV so that it becomes a boolean conditional RV; this can be obtained (in particular) by using a functional RV with an equality, e.g.
$$ \Pr(D = 6 \given D_1 = 1 \wedge D > 5) = \tfrac{1}{2}$$
As another example, we could now assume to have evidences on derived RV and query the explaining RVs: the following conditional RV
$$ D_1 \given D \le 3 $$ 
is characterized by the pmf $ \dpdf{(\text{1},\tfrac{2}{3}),(\text{2},\tfrac{1}{3}) } $. For this instance of causal inference, tuple RVs can bring up explanatory values, by revealing details of each atomic case; for instance, the pmf of the previous RV can easily be understood by tuples giving each die value and their sum:
$$ \tuple{D_1,D_2,D} \given D \le 3 $$ 
which has pmf $ \dpdf{(\tuple{1,1,2},\tfrac{1}{3}),(\tuple{1,2,3},\tfrac{1}{3}),(\tuple{2,1,3},\tfrac{1}{3}) } $. Here is a last example, which is more involved:
$$ D_1 \given D \in \{2,3,12\} \vee \lvert D_1-D_2 \rvert \ge 5 $$
which has pmf $ \dpdf{(\text{1},\tfrac{1}{2}),(\text{2},\tfrac{1}{6}),(\text{6},\tfrac{1}{3}) } $. As stated before, we see in all these examples the importance of referential consistency for getting the correct pmf.

A valid conditional RV $ X \given E $ requires that $X$ can produce at least one value verifying the condition expressed in $E$. This may be violated if $E$ is unfeasible. For instance, $ D_1 > 3 \given D_2 = D $ is invalid since $D_2 = D$ is certainly false (the sum of dice values is strictly greater than any die's). Let us point out that a statement like $ \Pr(D_1 > 3 \given D_2 = D) = 0 $ is not only wrong, it is actually meaningless: since we claim something that is contradictory, \emph{no} probability can be calculated. We shall see that such conditional RV shall be rejected as invalid by our algorithm.

Many probabilistic algorithms constrain evidence conditions to be \emph{observations}, which are equalities of the form $X = x$ or a conjunction of such equalities $X_1 = x_1 \wedge \ldots \wedge X_n = x_n $. Conditional RVs, as defined here, subsume this approach: they cover a far broader class of evidence conditions, for which the usual observations are just special cases. The sole constraint is to be able to express the evidences as a boolean function applying on some RVs; beyond equalities and conjunctions, this includes inequalities, negations, disjunctions, membership, etc.

\subsubsection{Table random variables} \label{sect:table_RV}

The fourth type of derived RV is defined by selecting a random variable in a lookup table based on the value taken by another RV.

\begin{defn}
Be a RV $C$ such that dom$(C) \triangleq \{c_1, \ldots, c_n\} $ with  $n \ge 1$ and be $n$ RV $X_1, \;\ldots\; , X_n$. The RV $T$ depending of $C$ such that, for any $i$,
$$C = c_i \;\; \implies \;\; T = X_i$$
is noted
$$ T \coloneqq \tablere {C} {c_1:X_1, \;\ldots\;, c_n:X_n} $$
and is called a \emph{table RV}.
\end{defn}

The order of RVs in the table is irrelevant. The table RVs allow defining conditional probability tables (CPT), which are used in Bayesian networks. Consider for example the well-known example of "Rain-Sprinkler-Grass" BN. We define three boolean RVs: $R$ represents whether it is raining, $S$ represents whether the sprinkler is on and $G$ represents whether the grass is wet. These RV have mutual dependencies as illustrated in the diagram \ref{fig:rsg_bn} (with the usual conventions of probabilistic graphical models):

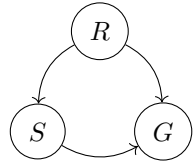
\begin{figure}[H]
\centering
\begin{tikzcd} [column sep=0.2em, row sep=0.8em, cells={nodes={circle,draw}}]
 & R \ar[ddl, bend right] \ar[ddr, bend left] & \\
 && \\
 S \ar[rr, bend right] && G
\end{tikzcd}
\caption{The Rain-Sprinkler-Grass Bayesian network}
\label{fig:rsg_bn}
\end{figure}

$R$ has a prior probability 0.20; the other probabilities and dependencies are quantified using CPTs: $S$'s probability depends of the weather: if it is raining the probability of $S$ is 0.01, otherwise it is 0.40; $G$ depends of both the weather and the grass state; the probabilities for $G$ depending of the values of tuple RV $\tuple{R,S}$ are: $\tuple{\false,\false}: 0.00 ,  \tuple{\true ,\false}: 0.80 , \tuple{\false,\true}: 0.90, \tuple{\true ,\true}: 0.99$. This BN can be modeled as follows: 
\begin{align*}
 R \coloneqq \rep{0.20} & \\
 S \coloneqq \tablere {R} { \true  &: \rep{0.01} , \\
                            \false &: \rep{0.40} } \\
 G \coloneqq \tablere {\tuple{R,S}}{
                                \tuple{\false,\false} &: \false , \\
                                \tuple{\true ,\false} &: \rep{0.80} , \\ 
                                \tuple{\false,\true}  &: \rep{0.90} , \\ 
                                \tuple{\true ,\true}  &: \rep{0.99} }
\end{align*}

Let us examine some queries we could make on this model using conditional RVs (since there are only boolean RVs, the $\Pr$ notation can be used): according to the definition of the table RV seen above, $ \Pr(S \given R) = 0.01$, $ \Pr(S \given \overline{R}) = 0.40$ and $ \Pr(G \given R \wedge \overline{S}) = Pr(G \given \eq{\tuple{R,S}}{\tuple{\true,\false}}) = 0.80$. Of course, the above results are just consistency check of the CPT, bringing no new information. As in any BN, the added value appear when evaluating $\Pr(G \given R) = 0.8019$ (forward chaining) or $\Pr(R \given G) = 0.3577$ (Bayesian inference). 

The table RVs allow modeling any CPT. Note that the number of entries shall be exactly equal to the cardinal of the domain of $C$. This can be cumbersome if this domain is large, e.g. if the condition is a tuple having many inner RVs (the domain of $C$ is the cartesian product of these RVs, provided that they are mutually independent). However, in several CPT, such as those having the property of \emph{contextual independence} \citep{pearl_1982,poole_zhang_2011}, redundancies can be avoided. To take an example, let us revisit the model above by assuming now that the probability to find the grass wet ($G'$) given that the sprinkler is on ($S$) equals 0.95, whatever it rained or not ($R$). Using the approach above naively, this probability 0.95 should be repeated on the last two clauses. Now, we can avoid such redundancy by defining the following cascaded tables construct
\begin{align*}
G' \coloneqq \tablere {S} { 
    \false :  \tablere {R} { \false &: \false ,     \\
                             \true  &: \rep{0.80} } \\ 
     \true :  \rep{0.95} }
\end{align*}
By avoiding redundancies in CPT, hence limiting the table size, it is easy to extrapolate the dramatic simplification gained on larger models. In the section dedicated to possible extensions (\ref{sect:extensions}), we shall present an extra type of derived RV, called \emph{mixture RV}, which offers an alternate way to express a CPT that also leverages contextual independence.

Another application of the table RV is the modeling of discrete-time Markov chains (DTMC). Consider for example the "Weather" DTMC represented in the following graph:
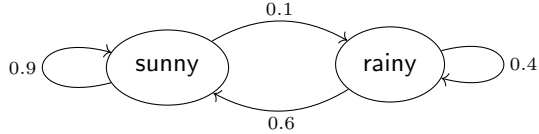
\begin{figure}[H]
\centering
\begin{tikzcd} [column sep=4em, cells={nodes={ellipse,minimum width=1.2cm,minimum height=1.0cm,draw}}]
 \mathsf{sunny} \ar[loop left,"0.9"] \ar[r, bend left,"0.1"] & \mathsf{rainy} \ar[l, bend left,"0.6"] \ar[loop right,"0.4"]
\end{tikzcd}
\caption{A discrete-time Markov chain}
\end{figure}

Assuming that the initial state $W_0$ at $t_0$ is sunny with probability 0.2,
$$ W_0 \coloneqq \re{\dpdf{(\mathsf{sunny},0.2), (\mathsf{rainy},0.8)}} $$
the future states can be modeled by table RVs $W_k$, defined by the following recurrence:
\begin{align*}
 W_{k+1} \coloneqq \tablere {W_k} {\mathsf{sunny} &: \re{\dpdf{(\mathsf{sunny} , 0.9) , (\mathsf{rainy} , 0.1)}}, \\
               \mathsf{rainy}  &: \re{\dpdf{(\mathsf{sunny} , 0.6) , (\mathsf{rainy} , 0.4)}}}
\end{align*}

\section{Exact marginalization inference} \label{sect:eval}

So far, we have defined different types of random variables, which could be interdependent. We have seen that the probabilistic primitives and compositional rules allow defining a large set of probabilistic models. The pmf of elementary RVs are known by definition. The pmf of derived RVs are initially unknown: the goal is now to find a general marginalization algorithm to calculate the pmf of any derived RV. This is the purpose of the Statues algorithm that will be presented in the present section.

The examples seen so far demonstrate that the evaluation of the pmf of a given RV can generally \emph{not} be performed by a simple recursive evaluation, as done for example in usual arithmetic. This constraint includes in particular arithmetic expressions like $X(X+Y)$ and all non-trivial conditional RV like $X \given X \le Y$. Actually, simple recursive evaluation is valid only if inner RVs are independent, that is if each RV occurs only once in the RV expression under evaluation. To obtain correct results in any case, the referential consistency shall be enforced on the top-level RV expression, which shall be considered as a whole. This calls for a dedicated structure for representing the RV dependencies.

\subsection{Representing RVs as \emph{p}-expressions} \label{sect:pex_dag}

The Statues algorithm requires as input a \emph{p-expression}, a structured object that defines the exact dependencies between random variables up to the elementary ones. For any given RV $X$, the associated \emph{p}-expression is noted $\dpd{X}$. If $X$ is a derived RV, then $\dpd{X}$ shall be decomposed into sub- \emph{p}-expressions up to elementary RVs, which are the atomic expressions. The table below provides the notations for \emph{p}-expressions primitives depending on the type of RV. 

\begin{center}
\begin{tabular}{r | c | c}
type & random variable $X$ & \emph{p}-expression $\dpd{X}$\\
\hline
&&\\[-8pt]
elementary & $A$ & $\pmf{A}$ \\
tuple & $\tuple{H \centerdot T}$ & $ \dpd{H} \otimes  \dpd{T} $ \\
functional & $f(X)$ & $ \efunc {f} {\dpd{X}}$ \\
conditional & $X \given E$ & $ \egiven{\dpd{X}}{\dpd{E}} $ \\
table & $\tablere {C} {c_1:X_1, \;\ldots\;, c_n:X_n}$ & $ \etable{\dpd{C}}{c_1:\dpd{X_1}, \;\ldots\;, c_n:\dpd{X_n} }$
\end{tabular}
\captionof{table}{\emph{p}-expressions by type of RV}
\label{tab:pex_types}
\end{center}


For the case of an elementary RV $A$, $\pmf{A}$ is meant to designate the pmf of $A$; this definition can be captured by the following relationship, for any pmf $\bp{a}$
$$ \pmf{\re{\bp{a}}\;} \triangleq \bp{a} $$
Any single value $v$ that is certain (i.e. non-random) is then associated to \emph{p}-expression $\pmf{v}$; this covers for example constant numerical values $\pmf{0}$, $\pmf{\pi}, ... \;$, booleans $\pmf{\true},\pmf{\false}$, as well as the empty tuple $\pmf{\tuple{}}$. For the case of a non-empty tuple RV, the rule is to recursively decompose the tuple into head element $H$ and tail tuple $T$; this means that the tuple RV $\tuple{ X_1, \;...\; , X_n }$ shall be associated, after iterations of the rule, to the \emph{p}-expression $X_1 \otimes \;...\; \otimes X_n$.\footnote{To be rigorous, the actual \emph{p}-expression should be $X_1 \otimes \;\ldots\; \otimes X_n \otimes \pmf{\tuple{}}$ since the empty tuple is the very last tail of any non-empty tuple (see \ref{sect:tuple_RV}). For the sake of conciseness, we shall omit the last operand in the present text and figures.} For the functional RV cases, it is important to make a distinction between notations $f$ and $\wideparen{f}$: the former is a usual unary function, the latter is just a part of the \emph{p}-expression $\efunc {f} {\dpd{X}}$ representing a distribution of values obtained by applying function $f$ to each possible value of $X$; also, as explained earlier, we only consider unary functions without losing generality: any $n$-ary function has a unary counterpart function that takes a $n$-tuple as argument.


Derived \emph{p}-expressions form recursive structures that can be represented as direct acyclic graphs (DAG). Figure \ref{fig:pex} shows the graphical convention used to represent the different types of \emph{p}-expression seen in table \ref{tab:pex_types} above. Note that the arrow direction, from parent node $P$ to child node $C$, is meant to represent that $P$ \emph{depends of} $C$.\footnote{One may deplore that this is the exact opposite of the convention used in probabilistic graphical models (see figure \ref{fig:rsg_bn}). Actually, we adopt here a point of view that is more suited for an algorithm: arrows represent references, as these are drawn for example in the trees representing arithmetic expressions.}

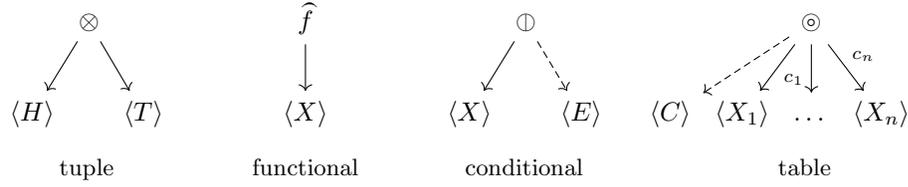
\begin{figure}[H]
    \centering
  \begin{subfigure} [h]{0.20\textwidth}
    \centering
    \begin{tikzcd} [column sep=0.05em]
        & \otimes \arrow[ld] \arrow[rd]&  \vphantom{\wideparen{f}} \\
        \dpd{H} &   & \dpd{T}
    \end{tikzcd}
    \caption*{tuple}
  \end{subfigure}
  \begin{subfigure}[h]{0.20\textwidth}
    \centering
    \begin{tikzcd} [column sep=0.05em]
       \wideparen{f} \arrow[d]  \\
       \dpd{X}
    \end{tikzcd} 
    \caption*{functional}
  \end{subfigure}
  \begin{subfigure}[h]{0.20\textwidth}
    \centering
    \begin{tikzcd} [column sep=0.05em]
      & \obar \ar[ld] \ar[rd,dashed] &  \vphantom{\wideparen{f}} \\
     \dpd{X} &                              & \dpd{E}
    \end{tikzcd} 
    \caption*{conditional}
  \end{subfigure}
  \begin{subfigure}[h] {0.20\textwidth}
    \centering
    \begin{tikzcd} [column sep=0.05em]
      && \circledcirc \vphantom{\wideparen{f}} \ar[lld,dashed] \arrow[ld,"c_1"] \arrow[d] \arrow[rd,"c_n"]&  \\
      \dpd{C} &\dpd{X_1} &  \ldots \vphantom{X}             & \dpd{X_n}
    \end{tikzcd} 
    \caption*{\hspace{5.em}table}
  \end{subfigure}
 \caption{DAG representation of \emph{p}-expressions}
 \label{fig:pex}
 \end{figure}

In the simplest cases, each RV occurs only once and the graph is a simple tree. For instance, consider the functional RV $X$ representing the fact that the added values of two dice is greater or equal to 6:

$$ D \coloneqq D_1 + D_2$$
$$ X \coloneqq D \ge 6$$

Then, the associated \emph{p}-expressions are:

$$ \dpd{D} =  \efunc {add} {\pmf{D_1} \otimes \pmf{D_2}}  $$
$$ \dpd{X} = \efunc {ge} {\dpd{D} \otimes \pmf{6}}  $$

which can be represented by the following tree:

\begin{figure}[H] 
\centering
\begin{tikzcd} [column sep=0.2em, row sep=0.8em]
&\dpd{X} = \hspace{-1em}& \wideparen{ge} \ar[d] & \\
&& \otimes \ar[ld] \ar[rd] & \\
& \wideparen{add} \ar[d] && \pmf{6}  \\
& \otimes \ar[ld] \ar[rd] &&   \\
\pmf{D_1} && \pmf{D_2} &
\end{tikzcd}
\caption{$\dpd{X}$ as a DAG}
\label{fig:pex-1}
\end{figure}
Now, consider the conditional RV $Y$:
$$ Y \coloneqq  D \given D \ge 6 \wedge D_2 \le 4 $$
The associated \emph{p}-expression is:
$$ \dpd{Y} = \egiven{\dpd{D}} {\efunc {and} {\dpd{X} \otimes \efunc{le}{\pmf{D_2} \otimes \pmf{4}}}} $$
By expanding this \emph{p}-expression, we see that $\dpd{D}$ and $\pmf{D_2}$ are referred twice. The associated graph is not a tree but a DAG (that is a more general type of graph).
\begin{figure}[H]
\centering
\begin{tikzcd} [column sep=0.2em, row sep=0.8em]
\dpd{Y} = \hspace{-1em}& \obar \ar[ddddd] \ar[rrrd,dashed] &&&&& \\
&&&& \wideparen{and} \ar[d] && \\
&&&& \otimes \ar[lld] \ar[rddd] && \\
&& \wideparen{ge} \ar[d] &&&& \\
&& \otimes \ar[ld] \ar[rd] &&&& \\
& \wideparen{add} \ar[d] && \pmf{6} &  & \wideparen{le} \ar[d] & \\
& \otimes \ar[ld] \ar[rd] &&  & \hphantom{\pmf{D_2}} & \otimes \ar[llld] \ar[rd] & \\
\pmf{D_1} && \pmf{D_2} &&&& \pmf{4}
\end{tikzcd}
\caption{$\dpd{Y}$ as a DAG}
\label{fig:pex-2}
\end{figure}

As a last example, here is the DAG for the table RV $S$ depending of elementary RV $R$ in the BN seen in \ref{sect:table_RV}.

\begin{figure}[H]
\centering
\begin{tikzcd} [column sep=0.6em, row sep=2.4em]
  && \dpd{S} = & \circledcirc \ar[llld,dashed] \arrow[ld,"\true"] \arrow[rd,"\false"]&  \\
      \pmf{R} && \pmf{\rep{0.01}} && \pmf{\rep{0.40}}
\end{tikzcd}
\caption{$\dpd{S}$ as a DAG}
\label{fig:pex-3}
\end{figure}
The DAG of table RV $G$ depending of RVs $R$ and $S$
shall not be represented here because too convoluted. Note that the four plain arcs should have been labelled with tuples $\tuple{\false,\false}$, $\tuple{\false,\true}$, etc.

\subsection{The Statues algorithm} \label{sect:statues_algo}

As we have seen, our probabilistic models are formalized as \emph{p}-expressions (\emph{pex} in the following); each such model is a DAG linking several \emph{p}-expressions together. The terminal pex correspond to elementary RV, which are defined by given pmf (also known as "prior probabilities"). 
The aim of the Statues algorithm is to calculate the exact probability mass function of a given pex.
The name \emph{Statues} is borrowed from the popular children's game of the same name\footnote{Other names include "Red Light, Green Light" (US), "Grandmother's Footsteps" (UK), "1-2-3, Soleil !" (France), "1-2-3, Piano !" (Belgium) and "Annemaria Koekkoek !" (Netherlands).}. The analogy with the algorithm should hopefully be clearer after the explanations given below.

The Statues algorithm uses a construction called \emph{generator}, which is a special case of coroutine \citep{knuth,saba}. Generators are available in several modern programming languages (e.g. C\#, Python, Ruby, Lua, Go, Scheme, ...), whether natively or as libraries. To state it in simple words, a generator is a special form of coroutine, which can suspend its execution to yield some object towards the caller and which can be resumed as soon as the caller has treated the yielded object. The generators are particularly well suited for combinatorial generation \citep{saba}. The reader may refer to appendix \ref{sect:generators} for a short introduction to generators and the related syntax used here.

For detailing the algorithm, we shall use the term \emph{atom} in the context of a given RV $X$ to designate a couple $(v,p)$ made up of a value $v$ and a probability $p$; an atom relates to a particular event that does not overlap with events related to other atoms. Such condition makes it possible to add without error the probabilities of atoms in a condensation treatment; more precisely, if we collect the $n$ atoms $(v,p_1), (v,p_2), ...,  (v,p_n)$ for value $v$, then $\Pr(X=v) = \Sigma _{i=1}^n p_i$. For instance, when throwing two fair dice, the probability to get the total 3 can be obtained by collecting the two atoms $(\tuple{1,2},\frac{1}{36})$ and $(\tuple{2,1},\frac{1}{36})$ for the tuple RV, then converting them to atoms $(3,\frac{1}{36})$ and $(3,\frac{1}{36})$ for the sum RV ; these two probabilities can then be added together, giving the expected result $\frac{1}{36} + \frac{1}{36} = \frac{1}{18}$.

The other important concept used in the algorithm is the \emph{binding}. At any stage of the execution, any given pex is either bound or unbound. At start-up, all pexes are unbound, which means that they have not yet been assigned a value. When a pex is required to browse the values of its domain, each yielded value is bound to the pex until the next value is yielded; when there are no more values, the pex is unbound. Once a pex is bound, it yields the bound value for any subsequent occurrence of this pex; the fact that a bound value is immobile for a while explains that it can be likened to a statue in the afore-mentioned game.

The Statues algorithm is made up of three parts. The entry-point is the subroutine $\textproc{marg}$, which takes a given \emph{p}-expression $d$ as argument and returns the marginalized pmf. This subroutine is not recursive but relies on $\textproc{genAtoms}$ and $\textproc{genAtomsByType}$ generators, which are mutually recursive. The calling graph is given on the figure \ref{fig:call_graph}.

\begin{figure} [H]
\centering
\begin{tikzcd} [row sep=3em]
\textproc{marg} \arrow[d, bend left, "\textit{calls}"] \\
\textproc{genAtoms} \arrow[d, bend left, "\textit{calls if unbound pex}"] \\
\textproc{genAtomsByType} \arrow[u, bend left, "\textit{calls if derived pex}"]
\end{tikzcd}
\caption{call graph of Statues algorithm}
\label{fig:call_graph}
\end{figure}
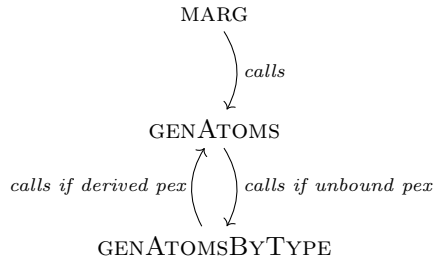

We shall present the algorithm in a top-down manner.  The entry-point $\textproc{marg}$ subroutine is given in algorithm  \ref{algo:statues1}.

\begin{algorithm} 
\caption{Statues algorithm -- part 1: $\textproc{marg}$ subroutine (entry-point)} \label{algo:statues1}
\begin{algorithmic}[1]
\Function{marg}{$d$}
\State $\beta \gets \{\}$ \Comment {init global binding store}
\State $a \gets \{\}$ \Comment {init unnormalized pmf}
\For {$(v,p) \gets \textproc{genAtoms}(d)$} \Comment {collect atoms}
	\If {$\nexists \; a[v]$}
		\State $a[v] \gets 0$ 
	\EndIf
	\State $a[v] \gets a[v] + p$ \Comment {condense pmf}
\EndFor
\If {$a = \{\}$}  \Comment {pmf is empty: error}
	\State \textbf{halt} with error
\EndIf
\State $s \gets \sum \limits_{(v,p) \in a} p $ \Comment {normalize pmf}
\State \Return $\dpdf {(v,\dfrac{p}{s}) \: \big\vert \: (v,p) \in a}$ 
\EndFunction
\end{algorithmic}
\end{algorithm}

$\textproc{marg}$ takes the given pex $d$ to be evaluated as argument. It invokes the $\textproc{genAtoms}$ generator and collects the atoms yielded one by one (line 4). We will see soon that $\textproc{genAtoms}$ uses the global binding store $\beta$, an associative array that is initially empty (line 2). Using the associative array $a$, the condensation (lines 5-8) treats atoms containing the same value so that they are merged together, by summing their probabilities. Once the $\textproc{genAtoms}$ generator is exhausted, a check verifies that at least one atom has been received (line 10), otherwise an error is reported and the subroutine halts (remember that we forbid empty pmf, as seen in \ref{sect:cond_RV}, this may occur if the evaluated pex is conditional while the given condition is impossible). The final step (line 13-14) normalizes the distribution $a$ to ensure that the probabilities of the pmf sum to 1 \footnote{It can be shown that the probability sum may differ from 1 only if the evaluated pex is conditional; actually, the performed division is closely related to the formula of conditional probability $\Pr(A \given C) \triangleq \dfrac {\Pr(A \wedge C)} {\Pr(C)}$. }. The pmf is then returned as a set of couples $(v_i,\Pr(d=v_i))$. 

\begin{algorithm} [H]
\caption{Statues algorithm -- part 2: $\textproc{genAtoms}$ generator } \label{algo:statues2}
\begin{algorithmic}[1]
\Procedure{genAtoms}{$d$}
	\If {$\exists \; \beta[d]$}			\Comment {$d$ is bound}
		\Yield $(\beta[d],1)$           \Comment {yield unique atom to caller}
	\Else								\Comment {$d$ is unbound}
		\For {$(v,\; p) \gets \textproc{genAtomsByType}(d)$}
			\State $\beta[d] \gets v$	\Comment {(re)bind $d$ to value $v$}
			\Yield $(v,\; p)$              \Comment {yield atom to caller}
		\EndFor
		\State delete $\beta[d]$		\Comment {unbind $d$}
	\EndIf
\EndProcedure
\end{algorithmic}
\end{algorithm}

The $\textproc{genAtoms}$ generator (algorithm  \ref{algo:statues2}) uses the binding store $\beta$ to check whether, in the current stage of the algorithm, the given pex is bound or not. If the given pex is not bound (lines 5-9), which is the case at least for the very first call on this pex, then $\textproc{genAtomsByType}$ is called and each atom yielded is bound to the pex before yielded in turn to the $\textproc{genAtoms}$'s caller. If the pex is bound (line 3), then the atom yielded is the bound value with probability 1; this behavior is actually the crux of the algorithm because it enforces the referential consistency.

\begin{algorithm} 
\caption{Statues algorithm -- part 3: $\textproc{genAtomsByType}$ generator} \label{algo:statues3}
\begin{algorithmic}[1]
\Procedure{genAtomsByType}{$d$} 
\\
\Switch{$d$}
\\
\Case {$\pmf{\re{\bp{a}}}$}			\Comment {$d$ is an elementary pex } 
		\For {$(v,\; p) \in \bp{a}$}
			\Yield $(v,\; p)$
		\EndFor
\EndCase
\\
\Case {$\efunc{f}{x}$}			\Comment {$d$ is a functional pex}
	\For {$(v,\; p) \gets \textproc{genAtoms}(x)$}
		\Yield $(f(v),\; p)$
	\EndFor
\EndCase
\\
\Case {$h \otimes t$}			\Comment {$d$ is a tuple pex}
	\For {$(v,\; p) \gets \textproc{genAtoms}(h)$}
	\For {$(s,\; q) \gets \textproc{genAtoms}(t)$}	
		\Yield $(\tuple{v \centerdot s},\; p . q)$
	\EndFor
	\EndFor
\EndCase
\\
\Case {$\egiven{x}{e}$}		\Comment {$d$ is a conditional pex}
	\For {$(v,\; p) \gets \textproc{genAtoms}(e)$}
		\If {$v$}
			\For {$(s,\; q) \gets \textproc{genAtoms}(x)$}
				\Yield $(s,\; p . q)$
			\EndFor
		\EndIf
	\EndFor
\EndCase
\\
\Case {$c \circledcirc g$}			\Comment {$d$ is a table pex}
	\For {$(v,\; p) \gets \textproc{genAtoms}(c)$}
		\For {$(s,\; q) \gets \textproc{genAtoms}(g[v])$}
			\Yield $(s,\; p . q)$
		\EndFor
	\EndFor
\EndCase

\EndSwitch
\\
\EndProcedure
\end{algorithmic}
\end{algorithm}

The $\textproc{genAtomsByType}$ generator (algorithm \ref{algo:statues3}) is the last part of the Statues algorithm. It yields the atoms according to the semantic of each type of pex. The dispatching is presented here as a pattern matching $\textbf{switch}$ construct although other constructs are feasible (see appendix \ref{sect:implem_design}). In the case of elementary pex, the treatment is simple and non-recursive. In the case of derived pex, the dependent pex shall be accessed by calling $\textproc{genAtoms}$ on them; this shall cause recursive calls, yielding atoms and updating the current bindings.
\begin{itemize}
\item
For elementary pex (lines 5-8), the atoms are simply the ones found in the pmf. 
\item
For functional pex $\efunc{f}{x}$ (lines 10-13), the treatment consists in applying the given function $f$ on the values of yielded atoms. As explained before, only unary functions are accepted; $n$-ary functions are emulated by functions having $n$-tuples as domain; then, the processing of tuple pex (see below) performs the required combinatorial on arguments.
\item
For tuple pex $h \otimes t$ (lines 15-20), the treatment consists in evaluating the combinatorial between head and tail values. Note that tuples having two or more elements are handled through recursive calls. The recursion halts when reaching the empty tuple, which is treated in the elementary pex case (i.e. singleton with probability 1).
\item
For conditional pex $\egiven{x}{e}$ (lines 22-29), the atoms of condition pex $e$ are collected one by one; for each atom containing the value $\true$, the treatment goes on and collects the atoms of the conditioned pex $x$. The atoms containing the value $\false$ are simply skipped; this bypass is important because it makes a pruning that prevents wasteful treatments: only the bindings verifying the given condition $e$ are retained.
\item
For table pex $c \circledcirc g$ (lines 31-36), the $g$ operand represents an associative array value-to-pex $\{c_1:x_1, \;\ldots\;, c_n:x_n \}$. The atoms of the key pex $c$ are collected one by one; for the value $v$ of each atom, the associated pex $g[v]$ is retrieved and the related atoms are collected in turn. 
\end{itemize}

To get a true understanding of the algorithm, one has to remember that $\textproc{genAtoms}$ and $\textproc{genAtomsByType}$ are \emph{not} subroutines returning a list of atoms; these are generators working cooperatively and yielding atoms one by one. At each yield, new bindings are created or removed. For instance, in the treatment of the conditional pex in $\textproc{genAtomsByType}$, the outer $\textbf{for}$ loop in line 23 makes some bindings that acts on inner $\textbf{for}$ loop in line 25: then, only atoms compatible with these bindings are yielded. Also, during algorithm execution, two generators ($\textproc{genAtoms}$ and $\textproc{genAtomsByType}$) are created for each node of the DAG; all these generators live together, the flow of control being changed at each \textbf{yield} instruction.

The correctness of the Statues algorithm is proved in appendix \ref{sect:proof}.

\subsection{Examples of execution}

To demonstrate how this algorithm works practically, we shall consider a couple of toy problems and trace the key steps of the execution (more involved use cases will be given in section \ref{sect:lea}).

\paragraph{Example 1}

We define a model that adds two Bernoulli RV $B_1$ and $B_2$, with respective probabilities $\tfrac{2}{3}$ and $\tfrac{1}{4}$.
\begin{align*}
B_1 & \coloneqq \re{\dpdf{(0,\tfrac{1}{3}), \; (1,\tfrac{2}{3})}} \\
B_2 & \coloneqq \re{\dpdf{(0,\tfrac{3}{4}), \; (1,\tfrac{1}{4})}} \\
S & \coloneqq B_1 + B_2
\end{align*}
The DAG of $\dpd{S}$ is displayed hereafter.
\begin{figure}[H]
\centering
\begin{tikzcd} [column sep=0.2em, row sep=0.8em]
& \wideparen{add} \ar[d] & \\
& \otimes \ar[ld] \ar[rd] & \\
\pmf{B_1} && \pmf{B_2}
\end{tikzcd}
\caption{$\dpd{S}$ as a DAG}
\end{figure}
The pmf of $S$ is calculated by invoking $\textproc{marg}(\dpd{S})$. The following table shows the sequence of steps executed by the algorithm. A step is defined by all the actions made by the main generator $\textproc{genAtoms}$ to yield a new atom (line 4 of $\textproc{marg}$). Each row shows some key data present or exchanged at a given step. The first two columns show the value bound on $\pmf{B_1}$ and $\pmf{B_2}$ during the given step. The remaining columns, labeled $C_{\hookrightarrow P}$ show atoms yielded by node $C$ to a parent node $P$ during the given step; this atom is the one yielded at line 7 of $\textproc{genAtoms}(C)$. The rightmost column $\wideparen{add}_{\hookrightarrow}$ shows the atom yielded by the main generator $\textproc{genAtoms}$: it is collected in $\textproc{marg}$, which is the final action of the step.

\begin{center}
\footnotesize
\renewcommand{\arraystretch}{1.8}
\begin{tabular}{r || c  c |  l  l  l  l }
 & $\pmf{B_1}$ & $\pmf{B_2}$ & $\pmf{B_1}_{\hookrightarrow \otimes}$ & $\pmf{B_2}_{\hookrightarrow \otimes}$ & $\otimes_{\hookrightarrow \wideparen{add}}$ & $\wideparen{add}_{\hookrightarrow}$ \\
\hline
\hline
\#1 & $0$ & $0$ & $(0,\tfrac{1}{3})$ & $(0,\tfrac{3}{4})$ & $(\tuple{0,0},\tfrac{1}{4})$ & $(0,\tfrac{1}{4})$  \\
\hline
\#2 & $0$ & $1$ & $(0,\tfrac{1}{3})$ & $(1,\tfrac{1}{4})$ & $(\tuple{0,1},\tfrac{1}{12})$ & $(1,\tfrac{1}{12})$ \\
\hline
\#3 & $1$ & $0$ & $(1,\tfrac{2}{3})$ & $(0,\tfrac{3}{4})$ & $(\tuple{1,0},\tfrac{1}{2})$ & $(1,\tfrac{1}{2})$  \\
\hline
\#4 & $1$ & $1$ & $(1,\tfrac{2}{3})$ & $(1,\tfrac{1}{4})$ & $(\tuple{1,1},\tfrac{1}{6})$ & $(2,\tfrac{1}{6})$
\end{tabular}
\captionof{table}{Trace of Statues algorithm on $\dpd{S}$}
\end{center}
Here are some explanations on this trace table.
When starting $\textproc{marg}(\dpd{S})$, generators $\textproc{genAtoms}$ /  $\textproc{genAtomsByType}$ will be created for each node of the DAG, in a top-down order until reaching the elementary pex $\pmf{B_1}$ and $\pmf{B_2}$. At step \#1, the embedded loops of tuple pex (line 16 of $\textproc{genAtomsByType}$) receive $(0,\tfrac{1}{3})$ yielded by $\pmf{B_1}$ and $(0,\tfrac{3}{4})$ yielded by $\pmf{B_2}$; it aggregates these into $(\tuple{0,0},\tfrac{1}{4})$, which is yielded to $\wideparen{add}$; then, the addition is applied on the received tuple (line 12 of $\textproc{genAtomsByType}$) and $(0,\tfrac{1}{4})$ is yielded to $\textproc{marg}$. At the steps \#2, \#3 and \#4, the embedded loops of tuple pex keep on enumerating all the combinations of values for $B_1$ and $B_2$, yielding atoms to $\wideparen{add}$ and following the same processing path as explained before; eventually, three more atoms are yielded to $\textproc{marg}$, viz. $(1,\tfrac{1}{12})$, $(1,\tfrac{1}{2})$ and $(2,\tfrac{1}{6})$. After completion of step \#4, the generators are exhausted; they terminate one after the other in the following order: $\pmf{B_2}$, $\pmf{B_1}$, $\otimes$ and $\wideparen{add}$. During this process, $\textproc{marg}$ has made the condensation of the four received atoms (i.e. merging the two $1$) into the associative array $a = \{ (0,\tfrac{1}{4}), (1,\tfrac{7}{12}), (2,\tfrac{1}{6}) \}$. The last step consists in normalizing $a$ to get the final pmf: each probability is divided by the sum $\tfrac{1}{4} + \tfrac{7}{12} + \tfrac{1}{6} = 1$ (in the present case, this is in unneeded because $a$ is already normalized). Eventually, the algorithm produces the final result
$$
 S \sim \dpdf{(0,\tfrac{1}{4}), (1,\tfrac{7}{12}), (2,\tfrac{1}{6})}
$$
which is correct. Note that the binding mechanism has been useless in this simple example: actually, we could have the same result by skipping the $\textproc{genAtoms}$ generator, e.g. replacing all the calls to it by calls to $\textproc{genAtomsByType}$. This is due to the fact that each RV appears no more than once in the queried RV; the referential consistency is then trivially verified. So, in the present example, the Statues algorithm merely performs a classical convolution. We shall see in the following examples the very role of the binding done by $\textproc{genAtoms}$.



\paragraph{Example 2}
To demonstrate the handling of referential consistency, we consider here the addition of a RV with itself:
\begin{align*}
R & \coloneqq B_1 + B_1
\end{align*}
where $B_1$ is defined as in the previous example. The DAG of $\dpd{R}$ is displayed hereafter.

\begin{figure}[H]
\centering
\begin{tikzcd} [column sep=0.2em, sep=0.8em]
\wideparen{add} \ar[d] \\
\otimes \ar[d,bend left] \ar[d,bend right] \\[1.4em]
\pmf{B_1}
\end{tikzcd}
\caption{$\dpd{R}$ as a DAG}
\end{figure}

Since the Statues algorithm enforces referential consistency, we shall legitimately expect that $R$ is equivalent to $2B_1$. Here is the trace of the execution of $\textproc{marg}(\dpd{R})$ following the same convention as before.

\begin{center}
\footnotesize
\renewcommand{\arraystretch}{1.8}
\begin{tabular}{r || c | l  l  l  l }
 & $\pmf{B_1}$ & $\pmf{B_1}_{\hookrightarrow \otimes}$ & $\pmf{B_1}_{\hookrightarrow \otimes}$ & $\otimes_{\hookrightarrow \wideparen{add}}$ & $\wideparen{add}_{\hookrightarrow}$ \\
\hline
\hline
\#1 & $0$ & $(0,\tfrac{1}{3})$ & $(0,1)$ & $(\tuple{0,0},\tfrac{1}{3})$  & $(0,\tfrac{1}{3})$  \\
\hline
\#2 & $1$ & $(1,\tfrac{2}{3})$ & $(1,1)$ & $(\tuple{1,1},\tfrac{2}{3})$  & $(2,\tfrac{2}{3})$
\end{tabular}
\captionof{table}{Trace of Statues algorithm on $\dpd{R}$}
\end{center}

In contrast with the previous example, the embedded loops of tuple pex (line 17 of $\textproc{genAtomsByType}$) both refer to the same pex, namely $\pmf{B_1}$; the outer loop receives the two atoms as before but the inner loop receives only one atom, containing the value bound at current step with the probability 1. The explanation lies in the binding mechanism, in the outer loop, $\pmf{B_1}$ is not yet bound (line 5 of $\textproc{genAtoms}$) while in the inner loop,  $\pmf{B_1}$ is bound (line 3 of $\textproc{genAtoms}$). Hence, the tuple pex yields only atoms with tuples having same inner values (i.e. $\tuple{0,0}$ and $\tuple{1,1}$). The final result is
$$
 R \sim \dpdf{(0,\tfrac{1}{3}), (2,\tfrac{2}{3})}
$$
that is the same as $2B_1$, as expected.

\paragraph{Example 3}
We shall elaborate example 1 to demonstrate conditional RV by querying the model under some given condition. Suppose we know (by whatever means) that the sum $S$ does not exceed 1; we want to get the pmf of $B_1$ given this evidence. This query can be modeled by a new RV $Q$ defined as follows:
\begin{align*}
Q & \coloneqq B_1 \given S \le 1
\end{align*}
which corresponds to the following DAG:
\begin{figure}[H]
\centering
\begin{tikzcd} [column sep=0.2em, row sep=0.8em]
\obar \ar[ddddd] \ar[rrd,dashed] &&& \\
&& \wideparen{le} \ar[d] & \\
&& \otimes \ar[ld] \ar[rd]& \\
& \wideparen{add} \ar[d] && \pmf{1} \\
& \otimes \ar[ld] \ar[rd] && \\
\pmf{B_1} && \pmf{B_2} &
\end{tikzcd}
\caption{$\dpd{Q}$ as a DAG}
\end{figure}

Here is the trace of the execution of $\textproc{marg}(\dpd{Q})$ following the same convention as example 1's.


\begin{center}
\footnotesize
\renewcommand{\arraystretch}{1.8}
\begin{tabular}{r || c  c |  l  l  l  l  l  l }
 & $\pmf{B_1}$ & $\pmf{B_2}$ & $\pmf{B_1}_{\hookrightarrow \otimes}$ & $\pmf{B_2}_{\hookrightarrow \otimes}$ & $\wideparen{add}_{\hookrightarrow \otimes}$ & $\wideparen{le}_{\hookrightarrow \obar}$ & $\pmf{B_1}_{\hookrightarrow \obar}$ & $\obar_{\hookrightarrow}$ \\
\hline
\hline
\#1 & $0$ & $0$ & $(0,\tfrac{1}{3})$ & $(0,\tfrac{3}{4})$ & $(0,\tfrac{1}{4})$ & $(\true,\tfrac{1}{4})$ & $(0,1)$ & $(0,\tfrac{1}{4})$ \\
\hline
\#2 & $0$ & $1$ & $(0,\tfrac{1}{3})$ & $(1,\tfrac{1}{4})$ & $(1,\tfrac{1}{12})$ & $(\true,\tfrac{1}{12})$ & $(0,1)$ & $(0,\tfrac{1}{12})$ \\
\hline
\#3 & $1$ & $0$ & $(1,\tfrac{2}{3})$ & $(0,\tfrac{3}{4})$ & $(1,\tfrac{1}{4})$ & $(\true,\tfrac{1}{2})$ & $(1,1)$ & $(1,\tfrac{1}{2})$ \\
\hline
\#4 & $1$ & $1$ & $(1,\tfrac{2}{3})$ & $(1,\tfrac{1}{4})$ & $(2,\tfrac{1}{6})$ & $(\false,\tfrac{1}{6})$ & $\;\;-$ & $\;\;-$
\end{tabular}
\captionof{table}{Trace of Statues algorithm on $\dpd{Q}$}
\end{center}

Since the root node is the condition pex, the first processing is the evaluation of the condition defined on the $\wideparen{le}$ node of the DAG (line 23 of $\textproc{genAtomByType}$). This shall cause the same processing as we have seen for example 1. What differs is that the atoms containing the sum are yielded to $\wideparen{le}$, which compares with value 1; the result $\true$/$\false$ is yielded to the $\obar$ node, which decides to continue the processing only if the $\true$ value is received. This happens for the first three steps; in such cases, a new generator is called on $\pmf{B_1}$ and, since $\pmf{B_1}$ is bound, the bound value is yielded with probability 1. For the step \#4, the condition is evaluated to $\false$, so nothing is yielded by the $\obar$ node for this step. During this process, $\textproc{marg}$ has made the condensation of the three received atoms into the associative array $a = \{ (0,\tfrac{1}{3}), (1,\tfrac{1}{2}) \}$. The last step consists in normalizing $a$ to get the final pmf:
$$ Q \sim \dpdf{(0,\tfrac{2}{5}), (1,\tfrac{3}{5})}$$
which is correct and, incidentally, different from the pmf of $B_1$: this shows that the given evidence does bring information on top of our prior beliefs.

We have seen in this last example how the treatment of conditional pex $\dpd{x} \obar \dpd{e}$ works: at each step, the evaluation of condition $\dpd{e}$ performs some bindings; for the steps where the condition is $\true$, $\dpd{x}$ is evaluated in turn \emph{taking into account these bindings}; so, the yielded values are guaranteed to verify the condition; for the steps where the condition is $\false$, $\dpd{x}$ is not queried, which makes a pruning in the search.

The three examples seen above are very basic use cases of the Statues algorithm. Actually, this algorithm is able to treat correctly far more involved probabilistic problems, in particular, all the examples given in sections \ref{sect:derived_RV} and section \ref{sect:lea}.

\section{Discussion} \label{sect:discussion}

As stated before, the Statues algorithm belongs to the category of exact probabilistic algorithms. The correctness of the algorithm is established by a proof (appendix \ref{sect:proof}); this proof uses invariants, formal specifications, propositional logic and basic probability theory. Beside this proof (and well before it), good confidence on the correctness has been gained through informal reasoning --~as sketched in the previous section~-- and, above all, by verifying the matching with results of problems found in the literature (see section \ref{sect:lea}).

The Statues algorithm, at its very heart, explore all possible paths or "possible worlds" \citep{de_raedt} compatible with given query. Without much surprise, it is limited by the NP-hard nature of inference on unconstrained BN \citep{cooper}. However, it performs far more efficiently than a naive inference by enumeration. We give hereafter three reasons to support this assertion. Firstly, since models and queries are DAG, the variables that do not impact the query at hand (i.e. those that are not are reachable by a directed path from the root query node) are not considered in the calculation; there is then a \emph{de facto} elimination of unused variables. Secondly, when looking for possible paths, the treatment of the conditional pex performs a pruning of the branches that do not comply with the given evidence; in many cases, this prevents wasteful calculations. Let us mention that a possible extension can even improve the pruning when the evidence is a conjunction of conditions (see the \emph{multi-conditional RV} in section \ref{sect:extensions}); this just requires a slight adaptation of the algorithm. Finally, since the binding done by $\textproc{genAtoms}$ is done for every RV, whether elementary or derived, it has the virtue of memoizing on the fly the results of functional pex, avoiding to redo the same calculation over and over. For instance suppose that, among a large set of RVs, a variable $D$ is defined as $D \coloneqq \sqrt{X^2 + Y^2}$; even if $D$ is used at multiple places of the query like in the expression $D^2 - U \times V \given (A \le D) \wedge (D \le B)$, the values of $D$ will be calculated \emph{only once} for each pair of values $[X,Y]$ and not for each combination of $[X,Y,A,B,U,V]$. This memoization is allowed without restriction since the functional pexes use, by definition, pure deterministic functions.

Due to the usage of generators, the execution model of the Statues algorithm is quite singular considering the large majority of algorithms based on subroutines. During algorithm execution, each pex involved in the evaluated query give rise to two generators, namely $\textproc{genAtoms}$ and $\textproc{genAtomsByType}$. These generators live together and their call graph mimics the query DAG, the yielded atoms being passed through the arcs of the DAG, from child node to parent node. As we have seen, each $\textproc{genAtomsByType}$ node performs a very simple treatment where probabilities are multiplied together. The collecting of atoms and their condensation are done at one place only, the root of the query DAG, that is the $\textproc{marg}$ subroutine (hence, the only place where probabilities are summed together). Unlike other exact algorithms, the Statues algorithm works with simple objects (the atoms) and simple operations on them (multiplication and sum of raw probabilities). In particular, there is no such things as \emph{factor} and \emph{pointwise product} like in variable elimination, there is no DAG transformation like in clustering algorithms. On the question of space complexity, the Statues algorithm is then expected to provide benefits comparable to those of the cutset conditioning method \citep{pearl_1988}.

Since the Statues algorithm is an exact probabilistic inference algorithm, let us briefly discuss the general merits and liabilities of calculating exact probabilities. As stated before, any exact probabilistic inference algorithm is limited in practice by the intractability of many problems, including large or densely connected BN; for such intractable problems at least, approximate algorithms like MCMC provide a fallback. Despite this constraint, the exact algorithms remain very useful for a number of reasons -- beside their exactness! Firstly, several problems \emph{can} be solved exactly in an acceptable time; this covers, at the very least, many sparsely connected BNs and the example cases used for education. Secondly, exact algorithms offer the opportunity to represent probabilities in different manners, beyond the prevailing floating-point numbers; probability representation as fractions enables perfect accuracy of results, tackling usual -- and annoying -- rounding errors. Additionally, symbolic computation is made possible by defining probabilities with variable symbols instead of numbers (e.g. $p, q, \;\ldots$) and by coupling the algorithm with a symbolic computation system; the output of $\textproc{marg}$ is then a pmf made up of probability expressions (e.g. $p^2(1-q)$). Such approach using probability symbols instead of numbers is useful when the same query is made over and over on the same complex model, with only varying prior probability values: the query result may be compiled offline once for all into an arithmetic expression (taking maybe a long processing time), then the resulting expression can be evaluated many times using fast arithmetic computations, with different input probability values.

Further research is definitely needed to factually assess the assets and liabilities of the Statues algorithm among the existing probabilistic inference algorithms. This includes at least the following research tracks:
\begin{itemize}
\item to make an objective comparison of the expressiveness of the underlying probabilistic framework with those used in other systems,
\item to study the complexity of the algorithm, both for space and time aspects, and to put these results in perspective with other comparable algorithms.
\end{itemize}

\section{Possible extensions} \label{sect:extensions}

We present here several possible extensions or improvements that can be added to the statues algorithm and its underlying framework.

The algorithm presents a drawback compared to other algorithms handling BN. If given evidences can be, fully or partially, expressed as a conjunctions of equalities on some elementary RVs, like $(X = x) \wedge (Y = y) \wedge ... $, then there is some waste of time for browsing X, Y, ... domains, evaluating equalities and eventually binding them to the sole values $x, y, ...$ Such wasteful process can be avoided easily by a pre-treatment before calling the $\textproc{marg}$ subroutine: elementary random variables like X, Y, ... may be bound explicitly and unconditionally on their respective values $x, y, ...$. Then, once $\textproc{marg}$ has completed, one can choose to unbind these variables or not, depending whether these bindings need to stand for next queries.

We have seen that our probabilistic models are made up of five building blocks, namely the RV types; these are summarized in table \ref{tab:pex_types} and treated individually by the $\textproc{genAtomsByType}$ generator. Although these five types have a broad scope in probabilistic modeling, it is possible, and even advisable, to add new pex types in order to enrich the expressiveness or allow for better algorithm performance. We shall briefly introduce hereafter three new types of RV. Their addition in our system is feasible without much difficulties: it would just require new notations and the adaptation of $\textbf{case}$ clauses in the $\textproc{genAtomsByType}$ generator, the rest of the algorithm remaining strictly unchanged.

The first new type of RV is called the \emph{multi-conditional RV}. The idea is to express a conjunction of given conditions as a sequence of boolean RVs: $X \given \tuple{C_1,\;\ldots\;,\;C_n}$. Actually, such conjunction is already expressible with the conditional RV presented in section \ref{sect:cond_RV} since the condition can be a functional RV defined by a logical AND: $X \given C_1  \wedge \ldots \wedge C_n$; however, with such approach, the Statues algorithm shall browse all the possible cases, even in the stages where some condition $C_i$ is bound to $\false$. The multi-conditional RV generalizes the conditional RV, allowing for optimizing the evaluation by pruning as soon as a condition of the given conjunction is bound to $\false$. Note that processing time can be saved by putting the most constraining conditions first or allowing assignment of priorities, as done in WebPPL \citep{dippl}. Similarities are worth pointing out here with constraint satisfaction problems (CSP) and the field of constraint programming.

The second new type of RV is called the \emph{multi-functional RV}. It is a generalization of the functional RV (see section \ref{sect:cond_RV}): instead of being defined with one single function, it accepts a RV having a set of functions as domain. So, a multi-functional RV not only randomizes the argument but it randomizes also the function to apply on this argument.

The third and last new type of RV is called  \emph{mixture RV}. Unlike, the two previous new types, the mixture RV is not a generalization of any RV presented so far. Basically, it is a RV choosing its value from a given set of RVs, which are equiprobable. A notation for such RV could be:
$$ \mixt{X_1, \;\ldots\; , X_n}$$
The most basic usage is to model "bag of dice" processes (first, draw a die from the bag, then, throw this die). A more advanced usage of mixture RV regards the CPT used in Bayesian networks. In section \ref{sect:table_RV}, we have seen that a CPT can be modeled as a table RV and that possible contextual independence can be leveraged to avoid redundancy through cascaded table RVs. Mixture RVs offer an alternate way to model CPT: each inner element is a conditional RV expressing a clause associating a condition with a resulting RV. By using this approach, several cases can be merged in the same clause. To illustrate the approach, here is how the revisited model $G'$ presented in \ref{sect:table_RV} can be modeled using mixture RV.
\begin{align*}
G' \coloneqq \rhd \bigl\{\hspace{16pt} \false \hspace{4pt} & \given \overline{R} \wedge \overline{S},\\
 \rep{0.80} & \given R \wedge \overline{S},\\
 \rep{0.95} & \given S \bigl\}
\end{align*}
Such kind of construct could prove to be effective also when the decision logic is more naturally expressed by conditions than by a lookup table. This happens in particular when the decision RV is numerical and can be divided in non-overlapping ranges; the following example assumes that $X$ is such RV, which influences a boolean RV $Y$:
\begin{align*}
Y \coloneqq \rhd \bigl\{ \rep{0.90}  & \given X < 30, \\
 \rep{0.80} & \given 30 \le X \; \wedge \;  X < 75,\;\\
 \rep{0.40} & \given 75 \le X \bigl\}
\end{align*}

As explained for cascaded table RVs, the present approach is sensible to avoid redundancies on large CPT. Note however that the algorithm shall have to evaluate all conditions one by one, which could be more demanding than in the table RV approach. Therefore, it is advisable to consider mixture RVs only for CPT where the number of redundant clauses is big enough. Mixing the two approaches within the same BN is of course feasible.

\section{Implementation -- Lea and MicroLea libraries} \label{sect:lea}

The Statues algorithm has been successfully implemented in the Python programming language \citep{vanrossum,psf}, namely in the Lea and MicroLea libraries that are introduced below. Python is well suited for the task because it natively supports generators \citep{pep255,saba}, dictionaries, operator-overloading and OO. Also, the "duck-typing" nature of Python \citep{foord} provides \emph{de facto} support to any RV domain, provided that it has a hashing function (for performing condensation through Python dictionaries).

We provide in appendix \ref{sect:implem_design} some general suggestions about the implementation of the Statues algorithm, whatever the programming language chosen.

\subsection{Lea library} \label{sect:lea_lib}

The prime implementation in Python is an open-source library called Lea \citep{lea}. Lea is fully workable, comprehensive and well documented; also it encompasses all the extensions presented in section \ref{sect:extensions}. 

It is worth pointing out that Lea, up to its version 2, stores probabilities as integer weights instead of commonly used floating-point numbers, as suggested in section \ref{sect:discussion}; this enables unlimited precision but the implementation requires special care when mixing distributions with different weight sums. Version 3 of Lea let the user choose between different types to represent probabilities, including fractions, decimal and floating-point numbers. Also, putting objects representing variable names in place of numerical probabilities brings up the capability to do symbolic computation, as discussed in section \ref{sect:discussion}; for this purpose, Lea 3 uses SymPy, a Python package dedicated to symbolic computation \citep{sympy}.

Let us mention that the understanding of the core marginalization algorithm is hard because \Lea's implementation contains several optimizations and extraneous functions, as standard indicators, information theory, random sampling, etc.; also, beside the Statues algorithm, Lea implements an approximation algorithm based on Monte-Carlo rejection sampling.

\subsection{MicroLea library} \label{sect:microlea_lib}

To help the understanding of the core algorithm, we have developed from scratch another open-source Python library: MicroLea, abbreviated as \mLea \citep{microlea}. \mLea is much smaller and much simpler than Lea: it focuses on the Statues algorithm and not more; also, it represents probabilities in the classical way, using floating-point numbers. The names of classes and methods match exactly the terminology used in the present paper. \mLea has a limited functionality and usability compared to Lea's but it is well suited to study how the Statues algorithm works.

As a short introduction to $\mu$Lea, we shall model the Rain-Sprinkler-Grass BN seen in section \ref{sect:table_RV} and we shall perform various queries on it. Here are the statements to instantiate this BN in $\mu$Lea:

\begin{samepage}
\begin{Verbatim}[fontsize=\small,xleftmargin=1cm]
from microlea import *

rain = ElemPex.bool(0.20)
sprinkler = TablePex( rain,
                    { True : ElemPex.bool(0.01),
                      False: ElemPex.bool(0.40)} )
grass_wet = TablePex( TuplePex(sprinkler, rain ),
                            { (False    , False): False,
                              (False    , True ): ElemPex.bool(0.80),
                              (True     , False): ElemPex.bool(0.90),
                              (True     , True ): ElemPex.bool(0.99)} )
\end{Verbatim}
\end{samepage}

Note that \mLea makes automatic conversion of fixed values into elementary pexes, when needed; this is why we can write \texttt{False} in place of \texttt{ElemPex.bool(0)} in the first entry of \texttt{grass\_wet}.

From these definitions, \mLea allows making several queries for which the $\textproc{marg}$ subroutine is called implicitly. Since Python is an interpreted language, the BN model can be queried in an interactive session, which is handy for experimenting. The method \texttt{given} builds a conditional pex from the boolean pex passed in argument; the operator-overloading is used to build functional pexes behind the scene for logical operators NOT (\texttt{\~}), AND (\texttt{\&}) and OR (\texttt{|}). Each query returns a pmf; when the pmf is boolean, the convenience function \texttt{P} is useful to extract the probability of \texttt{true}. Lines displaying the returned objects are indicated by a \texttt{\# ->} prefix. 
\begin{Verbatim}[fontsize=\small,xleftmargin=1cm,commandchars=\\\{\}]
sprinkler
\textit{# -> \{False: 0.6780, True: 0.3220\}}
P(sprinkler)
\textit{# -> 0.32200000000000006}
P(rain & sprinkler & grass_wet)
\textit{# -> 0.00198}
P(grass_wet.given(rain))
\textit{# -> 0.8019000000000001}
P(rain.given(grass_wet))
\textit{# -> 0.35768767563227616}
\end{Verbatim}
To check the consistency of these results, it is possible to retrieve the very last calculated probability thanks to the following expressions, which check respectively the definition of conditional probability and the Bayes' theorem:
\begin{Verbatim}[fontsize=\small,xleftmargin=1cm,commandchars=\\\{\}]
P(rain & grass_wet) / P(grass_wet)
\textit{# -> 0.35768767563227616}
P(grass_wet.given(rain)) * P(rain) / P(grass_wet)
\textit{# -> 0.35768767563227616}
\end{Verbatim}
Other relationships, including the axioms of probability and the chain rule, can be verified similarly in $\mu$Lea. Note that these relationships do not appear explicitly in the Statues algorithm; these are emerging properties of this algorithm.

As detailed before, functional pexes allow expressing more complex queries or evidences:
\begin{Verbatim}[fontsize=\small,xleftmargin=1cm,commandchars=\\\{\}]
P(rain.given(grass_wet & ~sprinkler))
\textit{# -> 1.0}
P(rain.given(~grass_wet | ~sprinkler))
\textit{# -> 0.27889355229430157}
P((rain | sprinkler).given(~grass_wet))
\textit{# -> 0.12983575649903917}
P((rain == sprinkler).given(~grass_wet))
\textit{# -> 0.87020050034444}
\end{Verbatim}

As an academic exercise, we can easily build the full joint probability distribution of the BN by using the tuple pex; this gives the probability of each atomic state of the three variables taking their interdependence into account:
\begin{Verbatim}[fontsize=\small,xleftmargin=1cm,commandchars=\\\{\}]
TuplePex(rain,sprinkler,grass_wet)
\textit{# -> \{(False, False, False): 0.4800, (False, True, False): 0.0320,}
\textit{      (False, True, True): 0.2880, (True, False, False): 0.0396,}
\textit{      (True, False, True): 0.1584, (True, True, False): 0.0000,}
\textit{      (True, True, True): 0.0020\}}
\end{Verbatim}
One can notice that there are only 7 entries in this joint probability distribution instead of the $2^3=8$ expected: this is due to the fact that the case \texttt{(False, False, True)} is impossible (see CPT of \texttt{grass\_wet}). Using such technique, it is possible to derive any joint probability distribution, whether full or partial, of any BN. This may provide useful clues to understand returned results.\footnote{In particular, one can build any intermediate \emph{factor}, as calculated by the variable elimination algorithm.}

To provide an example involving numerical RV, let us extend the BN with a device indicating a random value from 0 to 4; a CPT defines the pmf depending on the state of the grass:
\begin{Verbatim}[fontsize=\small,xleftmargin=1cm]
measure = TablePex( grass_wet,
                  { True : ElemPex({2: 0.125, 3: 0.375, 4: 0.500 }),
                    False: ElemPex({0: 0.500, 1: 0.375, 2: 0.125 })})
\end{Verbatim}
On this basis, we can freely mix booleans, numerical values and comparison operators in the same query:
\begin{Verbatim}[fontsize=\small,xleftmargin=1cm,commandchars=\\\{\}]
measure
\textit{# -> \{0: 0.2758, 1: 0.2069, 2: 0.1250, 3: 0.1681, 4: 0.2242\}}
measure.given(~rain)
\textit{# -> \{0: 0.3200, 1: 0.2400, 2: 0.1250, 3: 0.1350, 4: 0.1800\}}
P((measure <= 2).given(~rain))
\textit{# -> 0.685}
P(~rain.given(measure <= 2))
\textit{# -> 0.9018089662521034}
\end{Verbatim}
Finally, from the \texttt{measure} variable, we can derive a normalized value ranging from -1.0 to 1.0 and check the consistency with previous results:
\begin{Verbatim}[fontsize=\small,xleftmargin=1cm,commandchars=\\\{\}]
norm_measure = (measure-2.) / 2.
norm_measure.given(~rain)
\textit{# -> \{-1.0: 0.3200, -0.5: 0.2400, 0.0: 0.1250, 0.5: 0.1350, 1.0: 0.1800\}}
P(~rain.given(norm_measure <= 0.))
\textit{# -> 0.9018089662521034}
\end{Verbatim}

As a last example, we present a job scheduling problem with tasks having uncertain durations. There are 3 tasks to schedule: A, B and C. There is only one precedence constraint: task B shall not be started before the end of task A; we assume that there is enough resources to execute two tasks in parallel.

\begin{figure}[H]
\centering
\begin{tikzcd} [column sep=2.0em, row sep=0.2em]
                               & \framebox[2cm]{\centering {task A}} \ar[r] & \framebox[2cm]{\centering {task B}} \ar[rd] & \\
start \ar[ru] \ar[rd]  &                &                 & end \\
                               & \framebox[2cm]{\centering {task C}}  \ar[rru] & & 
\end{tikzcd}
\caption{Job scheduling example}
\end{figure}
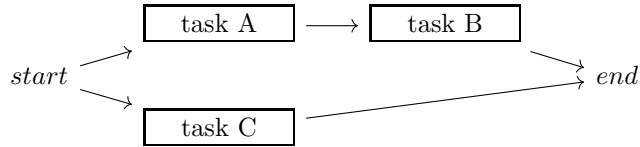

Durations of tasks A and B are characterized by known pmf (see below); duration of task C is conditioned by three possible scenarii, viz. CONSERVATIVE / EVOLUTIVE / DISRUPTIVE, each having a known probability to happen; the duration of task C is then modeled as a CPT giving a specific pmf for each scenario. We want then to calculate the shortest makespan and the total effort spent to complete the 3 tasks. 

\begin{Verbatim}[fontsize=\small,xleftmargin=1cm]
d_A = ElemPex({3: 0.1, 4: 0.8, 5: 0.1})
d_B = ElemPex({2: 0.5, 3: 0.5})
s = ElemPex({"CONSERVATIVE": 0.6, "EVOLUTIVE": 0.3, "DISRUPTIVE": 0.1})
d_C = TablePex(s,{"CONSERVATIVE": ElemPex({2: 0.7, 3: 0.3}),
                  "EVOLUTIVE"   : ElemPex({3: 0.5, 4: 0.5}),
                  "DISRUPTIVE"  : ElemPex({7: 0.2, 8: 0.7, 9: 0.1 })})
makespan = FuncPex(max,TuplePex(d_A+d_B,d_C)) 
efforts = d_A + d_B + d_C
\end{Verbatim}

The makespan has been defined by evaluating the duration of the critical path; this uses a functional pex that applies Python's \texttt{max} function on the two possible paths. From this probabilistic job scheduling model, one can now make several queries to calculate probability distributions of makespan and efforts, possibly integrating new information, assumptions or constraints.

\begin{Verbatim}[fontsize=\small,xleftmargin=1cm,commandchars=\\\{\}]
makespan
\textit{# -> \{5: 0.0450, 6: 0.4050, 7: 0.4240, 8: 0.1160, 9: 0.0100\}}
makespan.given(s == "CONSERVATIVE")
\textit{# -> \{5: 0.0500, 6: 0.4500, 7: 0.4500, 8: 0.0500\}}
makespan.given(s != "CONSERVATIVE")
\textit{# -> \{5: 0.0375, 6: 0.3375, 7: 0.3850, 8: 0.2150, 9: 0.0250\}}
efforts.given((s == "DISRUPTIVE") & (efforts <= 14))
\textit{# -> \{12: 0.0183, 13: 0.2294, 14: 0.7523\}}
makespan.given(efforts == 8)
\textit{# -> \{5: 0.0803, 6: 0.9197\}}
\end{Verbatim}

Although not conventional, backward reasoning may be done also to infer explanations from posterior  measures, assuming that some causal variables (namely, the scenario and/or specific task durations) remain uncertain.

\begin{Verbatim}[fontsize=\small,xleftmargin=1cm,commandchars=\\\{\}]
s.given((makespan <= 7) & (efforts <= 9))
\textit{# -> \{CONSERVATIVE: 0.8556, EVOLUTIVE: 0.1444\}}
s.given(makespan == 9)
\textit{# -> \{DISRUPTIVE: 1.0000\}}
TuplePex(d_A,d_B,d_C).given((makespan == 6) & (s == "CONSERVATIVE"))
\textit{# -> \{(3, 3, 2): 0.0778, (3, 3, 3): 0.0333, (4, 2, 2): 0.6222, (4, 2, 3): 0.2667\}}
TuplePex(d_A,d_B,d_C).given((makespan == 5) & (s == "CONSERVATIVE"))
\textit{# -> \{(3, 2, 2): 0.7000, (3, 2, 3): 0.3000\}}
\end{Verbatim}

We see in all these examples that the different pex types can be used and composed together to make expressive probabilistic models and queries, following the idea of probabilistic programming. There are many other examples and use cases provided on \mLea and Lea project pages (the syntax and output format in Lea are slightly different). These examples cover, among others, probabilistic arithmetic, which integrate seamlessly with conditional probabilities and Bayesian reasoning.

\section{Conclusions}

In the present paper, we have introduced a framework, namely the \emph{p}-expressions, that is meant to cover several probabilistic modeling techniques for discrete random variables having a finite domain. In essence, this framework provides primitives to define probabilistic models as direct acyclic graphs capturing the dependencies between random variables up to given prior probability mass functions. We have shown through examples that this formalism appears to be rich enough to model probabilistic arithmetic, conditioning, discrete-time Markov chains and Bayesian networks.

We have then presented a new algorithm, the Statues algorithm, which makes exact marginalization inference on those models. This algorithm relies on a special binding mechanism that uses recursive generators. Some simple examples have been provided to show this algorithm in action. In the last part, we have presented the Lea and MicroLea libraries, which are two implementations of the Statues algorithm in Python. The usage of MicroLea has been demonstrated in a set of nontrivial examples, including the definition of a Bayesian network and the treatment of advanced queries on it.

The merits and liabilities of the Statues algorithm have been shortly discussed, as well as possible extensions. The algorithm handles only discrete elements and it does not overcome the computational limitations of exact probabilistic inference. However, one of its interests in the perspective of probabilistic programming resides in its ability to address a set of problems traditionally handled by different specialized probabilistic modeling approaches. Also, it provides several  gains of efficiency compared to other exact algorithms, notably through its pruning and memoization features. For the algorithm's inner machinery, the binding mechanism based on recursive generators has proven to be elegant and powerful to handle the dependencies between random variables. The interface of the algorithm is formally specified and its correctness is proved (appendix \ref{sect:proof}). Despite these promising results, further work is needed to factually assess its assets/liabilities among other probabilistic inference algorithms.

\section{Acknowledgments}

The author warmly thanks Nicky van Foreest for reviewing the first version of the present paper and for providing fruitful advices to improve it. The author is grateful to Fr\'ed\'eric and Marie-Astrid Buelens for their wise recommendations about writing a scientific paper. The author thanks Gilles Scouvart, Nicky van Foreest, Zhibo Xiao, Noah Goodman, Paul Moore, Thomas Laroche and Guy Lalonde for their feedback, support, suggestions or contributions provided for the Lea library. The author thanks Guy Van den Broeck for having endorsed him to submit on arXiv/cs.MS.

This research did not receive any specific grant from funding agencies in the public, commercial, or not-for-profit sectors.

\newpage
\bibliographystyle{authordate1}
\bibliography{statues_algorithm_v2}

\newpage

\begin{appendices}

\section{Programming with generators} \label{sect:generators}

The concept of \emph{generator} in programming is linked to those of \emph{subroutine} and \emph{coroutine}. We assume that the prevalent notion of subroutine (known also as subprogram, function or procedure) does not require further explanation. \emph{Coroutines} are generalizations of subroutines that allow for multiple entry points, that can yield multiple times, and that resume their execution when called again \citep{saba}. \emph{Generators} are special cases of coroutines in the sense that they are constrained to yield values to the caller only.
 
For introducing the idea practically, let us consider the following example where we define a generator $(\textproc{genMessages}$) and a subroutine ($\textproc{main}$) that calls this generator.

\begin{algorithm}[H]
\caption{ basic example of generator } 
\begin{algorithmic}[1]

\Procedure{genMessages}{\,}
	\Yield $\texttt{"1, 2, 3!"}$
	\For {$msg \gets \tuple{\texttt{"Red light!"}, \; \texttt{"Green light!"}}$}
		\Yield $msg$
	\EndFor
\EndProcedure
\\
\Function{main}{\,}
    \State display $\texttt{"starting..."}$
	\For {$msg \gets \textproc{genMessages()}$}
		\State display $\texttt{"received: "}, msg$
	\EndFor
    \State display $\texttt{"end"}$
\EndFunction

\end{algorithmic}
\end{algorithm}

When calling $\textproc{main}$, the $\textproc{genMessages}$ is called and it yields three textual messages, one by one, to $\textproc{main}$'s loop. Just after each $\textbf{yield}$ statement, $\textproc{genMessages}$ freezes and gives back the control to $\textproc{main}$; at each \textbf{for} loop iteration of $\textproc{main}$, the generator $\textproc{genMessages}$ is resumed just after the $\textbf{yield}$. Here is the output:
\\

\noindent
$\texttt{starting...}$ \\
$\texttt{received: 1, 2, 3!}$ \\
$\texttt{received: Red light!}$ \\
$\texttt{received: Green light!}$ \\
$\texttt{end}$ \\

This trace clearly shows that the generator and its caller work in close cooperation, the  execution being interleaved, the control flow going back and forth between the two. This contrasts with a usual subroutine, which returns one single result to the caller. Note that a practical indicator to identify a generator is the presence of \textbf{yield} statements in its body.

To elaborate the idea, let us present a less contrived example, which demonstrates the use of recursive generators. Consider the following problem:

\begingroup
\noindent \leftskip5em \rightskip4em
-- Given two integers $n$ and $k$ such that $0 \le k \le n$, which are the binary words of $n$ bits having exactly $k$ bits equal to one?
\par
\endgroup

To give an example of results that we expect, here is the list of 4 bits-long words having exactly 2 ones ($n = 4$, $k = 2$): 
$$\tuple{0,0,1,1},\tuple{0,1,0,1},\tuple{0,1,1,0},\tuple{1,0,0,1},\tuple{1,0,1,0},\tuple{1,1,0,0}$$
A recursive algorithm to solve this problem for $(n,k)$ consists in $1^{\circ}$ assuming that the first bit is "0", make a recursive call with $(n-1,k)$ to concatenate with words of length $n-1$ and $k$ occurrences of "1", then $2^{\circ}$ assuming that the first bit is "1", make a recursive call with $(n-1,k-1)$ to concatenate with words of length $n-1$ and $k-1$ occurrences of "1"; the recursion eventually stops  when reaching the trivial case $n=0$: if $k = 0$ then an empty word is yielded, otherwise (if $k \neq 0$) then nothing is yielded since there cannot exist any bit in an empty word. The generator $\textproc{genWords}$ below formalizes this algorithm; the subroutine $\textproc{main}$ calls $\textproc{genWords}$ for solving the actual problem instance given above.

\begin{algorithm}[H]
\caption{ example of recursive generator  } 
\begin{algorithmic}[1]

\Procedure{genWords}{$n,k$}
	\If {$n = 0$}
		\If {$k = 0$}
			\Yield $\tuple{}$
		\EndIf
	\Else 
		\For {$w \gets \textproc{genWords}(n-1,k)$}
			\Yield $\tuple{0 \centerdot w}$
		\EndFor
		\For {$w \gets \textproc{genWords}(n-1,k-1)$}
			\Yield $\tuple{1 \centerdot w}$
		\EndFor
	\EndIf
\EndProcedure
\\
\Function{main}{\,}
	\For {$w \gets \textproc{genWords}(4,2)$}
		\State display $w$
	\EndFor
\EndFunction

\end{algorithmic}
\end{algorithm}
Invoking the $\textproc{main}$ subroutine displays the list of 6 words given above, in the same order. The execution of a recursive generator is very different from the execution of a usual recursive subroutine. In the present case, at each iteration of $\textproc{main}$'s loop, we have a chain of 5 active $\textproc{genWords}$ generators, each of which being blocked just after a \textbf{yield} statement.

\newpage
\section{Hints on implementation} \label{sect:implem_design}

The Statues algorithm is expressed in a highly abstract way, in order to be short and language-neutral. We aim here to provide some suggestions for implementing this algorithm in an actual program. We assume that the target programming language natively supports recursive generators and associative arrays\footnote{also known as (\emph{hash}) \emph{maps}, \emph{hash tables}, \emph{dictionaries} or \emph{symbol tables}}, to ease the pmf condensation at least; if this is not the case, equivalent constructs shall be implemented or installed from add-on packages.

\begin{enumerate}
\item If the programming language supports object orientation, it is advisable to define the different types of pex in dedicated classes. These classes should inherit from one abstract class representing any pex. The DAG structure can be captured using the \emph{composite} design pattern \citep{gof}. Then, following the \emph{template method} design pattern, $\textproc{marg}$ and $\textproc{genAtoms}$ can be defined as methods of the pex abstract class; the $\textproc{genAtomsByType}$ generator is then distributed into short methods in each pex concrete subclass, so the big $\textbf{switch}$ statement can be avoided.

\item The referential consistency, a key aspect of the algorithm, requires that pex instances are assembled in a DAG without duplication of instances (see section \ref{sect:pex_dag}). In any pex, each occurrence of a name shall refer to the same object. This may be achieved easily by using references or pointers.

\item There are several ways to implement the binding store $\beta$, which is a mutable object that must be shared by all pexes. As suggested by the algorithm, $\beta$ could be implemented by using an associative array accessible as a global variable (the keys could be, for example, variable names or references to pex instances). If an OO approach is chosen (as sketched above), this binding store could be removed: the bound value could be stored as an attribute at the level of each pex instance; a flag or a special dummy value should then be defined to represent the "no value bound" state. Note however that all these approaches entail that the routines involved in the algorithm are not reentrant, hence these are not suited to evaluate in parallel several queries sharing elements. A solution to enable such parallelism is to pass the binding store as argument to the routines.\footnote{This last approach is quite similar to Church's, with the concept of \emph{environment} \citep{goodman}}. This approach is elaborated in the proof provided in appendix \ref{sect:proof}, which uses an updated algorithm (see \ref{sect:algo_bis}).

\item In some cases, functional pex can cause errors (e.g. division by zero), which prevents resuming the calculation. In many language, this may raise an exception that halts the processing. In such situation, depending on the implementation of the binding store, it may be important to unbind all pexes that are currently bound. The $\textbf{finally}$ block, present in several languages, can be useful for unbinding pex in the $\textproc{genAtoms}$ generator.

\item If operator overloading is supported by the programming language, the expressiveness of functional pexes may be greatly improved by redefining infix operators for arithmetic ($+,-,*,/,\,\ldots$), comparison ($==,<,<=,\ldots$) and logic (not, and, or, \ldots).

\item The pmf for elementary pexes do not need a direct access through dictionary. A compact data structure allowing sequential access is sufficient for the algorithm.

\item A lot of functional pexes cover 2-ary functions. The performance of the algorithm can be slightly improved by having a dedicated functional pex, which does not rely on tuple RVs to make the combinatorics. Such case shall be handled by a dedicated case in the $\textproc{genAtomsByType}$ generator, with two embedded \texttt{for} loops.

\item The binding/unbinding performed by $\textproc{genAtoms}$ is unconditional. In several cases however, such handling is superfluous; actually, this treatment could be skipped in two cases: for singleton elementary pexes and for any pex which is referred only once in the evaluated pex. An optimized implementation could then detect such pexes in a preprocessing and flag them to skip their binding.
\end{enumerate}

For the readers eager to see concrete application of these hints in Python, let us mention that the \mLea implementation (section \ref{sect:microlea_lib}) follows the points 1, 2, 3, 4 (binding value attribute), 5 and 6 presented above. The Lea implementation (section \ref{sect:lea_lib}) follows all the eight hints but, as already explained, is harder to understand.

\newpage
\section{Proof of algorithm correctness} \label{sect:proof}
 
We provide in the present appendix a proof of the correctness of the Statues algorithm. This proof uses propositional logic and basic probability theory.

\subsection{Conventions and definitions}

In the following, an uppercase latin letter, like $X, Y, ...$ , represents a random variable (as defined in \ref{sect:RV}), a lowercase latin letter, like $x, y, ...$ , represents a value or a pex (as defined in \ref{sect:pex_dag}) and a lower case greek letter, like $\alpha, \lambda, ...$ represents a logical proposition. By convention, the equality operator has precedence over logical connectors; for instance, the proposition $\alpha \wedge \; X = x \; \wedge \beta$ is meant for $\alpha \wedge (X = x) \wedge \beta$. The following notation
$$ \Or{i}{\text{cond}(i)}{\lambda_i} $$
is defined as a disjunction of a subset of $n$ propositions $\{\lambda_i\}$, each of which is present only if some condition cond$(i)$ is true. If $n=0$ or if cond$(i)$ is false for any $i$, then this expression is defined to be equivalent to $\false$. 

We shall use the following definitions.
\begin{defn} \label{DEF_MUTEX}
The set of propositions $ \{ \lambda_i \} $ are \emph{mutually exclusive} iff
$$ \forall \; i, j : \quad \lambda_i \wedge \lambda_j \quad \Rightarrow \quad i = j$$
\end{defn}

\begin{defn}
A \emph{binding assertion} $\alpha$ is a logical proposition of the form
$$ \alpha \; \triangleq \; \bigwedge_{i} \big( X_i = x_i \big) $$
for any set of $n$ random variables $X_i$ and any set of $n$ values $x_i \in \text{dom}(X_i)$ such that this conjunction is not contradictory, i.e. $\PR{\alpha} > 0$. A state expressing no binding (in particular, at algorithm start-up) is indicated by the binding assertion ${\alpha = \true}$.
\end{defn}

\begin{defn}
Be a binding assertion $\alpha$ and a pex $d \triangleq \dpd{X}$ for some random variable $X$.  We say that $\alpha$ \emph{binds} $d$ (or $\alpha$ \emph{binds} $X$) iff $X$ appears in some equality present in $\alpha$. We have then 
$$ \alpha \quad \Rightarrow \quad X = x $$
for some value $x \in \text{dom}(X)$.
\end{defn}

With such definitions, the state of the binding store, i.e. the binding of any subset of random variables at any given state of the execution of the algorithm, can be expressed as a binding assertion.

\subsection{Adding invariants to the algorithm} \label{sect:algo_bis}

The Statues algorithm involves one entry-point subroutine (\textproc{marg}) and two generators \textproc{genAtoms} and \textproc{genAtomsByType}. To establish a formal proof of the correctness of \textproc{marg}, we need first to specify formally what is meant to be correct for \textproc{marg} and, incidentally, the same for \textproc{genAtoms} and \textproc{genAtomsByType}. Such specifications are difficult to provide on the algorithm in the form presented in \ref{sect:statues_algo} because generators cannot be specified as pre- and post-conditions as done for usual subroutines; also, these generators use a binding store, namely $\beta$, that does not appear in their signature and which is updated at each step of execution. To ease the specifications of the generators and the proof of their correctness, we shall first rewrite the algorithm to make the binding explicit at each step of the execution.\footnote{As stated in appendix  \ref{sect:implem_design}, this rewriting may be interesting also to obtain a true \emph{functional programming} style, removing the need of any mutable object.} Below, we have augmented each of the two generators \textproc{genAtoms} and \textproc{genAtomsByType} presented in algoritms \ref{algo:statues2} and \ref{algo:statues3}, so that the binding assertions are made explicit by means of invariants. First, an argument $\alpha$ has been added to each generator's signature: it indicates the binding assertion that holds during the whole execution of that generator. Secondly, each yielded atom contains a third element, which is a new binding assertion $\lambda$ to add to the current binding assertion $\alpha$; this $\lambda$ assertion holds until the next \textbf{yield} statement of the generator or until its end.

\begin{algorithm} [H]
\caption{Statues algorithm -- part 2 (rewritten): $\textproc{genAtoms}$ generator } 
\label{algo:statues2b}
\begin{algorithmic}[1]
	\Procedure{genAtoms}{$d,\alpha$}   \Comment {assuming $d \triangleq \dpd{Y}$}
	\If {$ \alpha \Rightarrow Y = v $}		\Comment {$d$ is bound}
		\Yield $(v,\; 1,\; \text{true})$   \Comment {yield unique atom to caller}
	\Else								\Comment {$d$ is unbound}
		\For {$(v,\; p,\; \sigma) \gets \textproc{genAtomsByType}(d,\alpha)$}
			\Yield $(v,\; p,\; \sigma \wedge Y = v)$       \Comment {yield atom to caller}
		\EndFor
	\EndIf
\EndProcedure
\end{algorithmic}
\end{algorithm}

Here are some explanations on the specific example of algorithm \ref{algo:statues2b} to help understanding of these invariant notations. The given binding condition $\alpha$ is true on the whole execution of \textproc{genAtoms} (from line 2 to 8), even if the control is passed back to the caller at \textbf{yield} statements (lines 3 and 6); this assertion is enforced in line 5, by passing $\alpha$ to \textproc{genAtomsByType}. The $\sigma$ present in each yielded atom is a binding assertion giving possible new binding(s) made by \textproc{genAtomsByType}; the $\sigma$ assertion is an invariant that is true during the current loop iteration, i.e. the line 6; hence, the atom yielded on line 6 contains $\sigma \wedge Y = v$, which encompasses the current binding done by \textproc{genAtomsByType} and the current binding made on $Y$ by \textproc{genAtoms}.

\begin{algorithm} [H]
\caption{Statues algorithm -- part 3 (rewritten): $\textproc{genAtomsByType}$ generator } 
\label{algo:statues3b}
\begin{algorithmic}[1]
\Procedure{genAtomsByType}{$d,\alpha$} 
\\
\Switch{$d$}
\\
\Case {$\pmf{\re{\bp{a}}}$}			\Comment {$d$ is an elementary pex } 
		\For {$(v,\; p) \in \bp{a}$}
			\Yield $(v,\; p,\; \true)$
		\EndFor
\EndCase
\\
\Case {$\efunc{f}{x}$}			\Comment {$d$ is a functional pex}
	\For {$(v,\; p,\; \lambda) \gets \textproc{genAtoms}(x,\alpha)$}
		\Yield $(f(v),\; p,\; \lambda)$
	\EndFor
\EndCase
\\
\Case {$h \otimes t$}			\Comment {$d$ is a tuple pex}
	\For {$(v,\; p,\; \lambda) \gets \textproc{genAtoms}(h,\alpha)$}
	\For {$(s,\; q,\; \mu) \gets \textproc{genAtoms}(t,\alpha \wedge \lambda)$}	
		\Yield $(\tuple{v \centerdot s},\; p.q ,\; \lambda \wedge \mu)$
	\EndFor
	\EndFor
\EndCase
\\
\Case {$\egiven{x}{e}$}		\Comment {$d$ is a conditional pex}
	\For {$(v,\; p,\; \lambda) \gets \textproc{genAtoms}(e,\alpha)$}
		\If {$v$}
			\For {$(s,\; q,\; \mu) \gets \textproc{genAtoms}(x,\alpha \wedge \lambda)$}
				\Yield $(s,\; p.q ,\; \lambda \wedge \mu )$
			\EndFor
		\EndIf
	\EndFor
\EndCase
\\
\Case {$c \circledcirc g$}			\Comment {$d$ is a table pex}
	\For {$(v,\; p,\; \lambda) \gets \textproc{genAtoms}(c,\alpha)$}
		\For {$(s,\; q,\; \mu) \gets \textproc{genAtoms}(g[v],\alpha \wedge \lambda)$}
			\Yield $(s,\; p.q,\; \lambda \wedge \mu)$
		\EndFor
	\EndFor
\EndCase

\EndSwitch
\\
\EndProcedure
\end{algorithmic}
\end{algorithm}

To be consistent, the \textproc{marg} subroutine (algorithm \ref{algo:statues1}) shall be slightly updated also: the binding condition $\true$ shall be added as \textproc{genAtoms} argument; this is meant to declare that there is no binding at start-up; this replaces  the initialization of empty global store $\beta$ in the initial algorithm.


\begin{algorithm} 
\caption{Statues algorithm -- part 1 (rewritten): $\textproc{marg}$ subroutine} \label{algo:statues1b}
\begin{algorithmic}[1]
\Function{marg}{$d$}
\State $a \gets \{\}$ \Comment {init unnormalized pmf}
\For {$(v,p,\lambda) \gets \textproc{genAtoms}(d,\true)$} \Comment {collect atoms}
	\If {$\nexists \; a[v]$}
		\State $a[v] \gets 0$ 
	\EndIf
	\State $a[v] \gets a[v] + p$ \Comment {condense pmf}
\EndFor
\If {$a = \{\}$}  \Comment {pmf is empty: error}
	\State \textbf{halt} with error
\EndIf
\State $s \gets \sum \limits_{(v,p) \in a} p $ \Comment {normalize pmf}
\State \Return $\dpdf {(v,\dfrac{p}{s}) \: \big\vert \: (v,p) \in a}$ 

\EndFunction
\end{algorithmic}
\end{algorithm}

\subsection{Formal specifications}

In order to establish the proof of correctness, we have to provide a specification formalizing the pre- and post-conditions of the \textproc{marg} subroutine, the entry-point of the algorithm.

\begin{defn} \label{MA_DEF}
Be a pex $d \triangleq \dpd{Y}$ for some random variable $Y$. The subroutine \textproc{marg}$(d)$ is correct iff it terminates \\
-- either by reporting an error if dom$(Y)$ is empty \\
-- or by returning a pmf $\bp{r}$ such that the following two conditions hold:
\vspace{-12pt}
\begin{addmargin}[4em]{0em}
\begin{flalign*}
& \left\{
    \begin{array}{ll}
 \text{dom}(\bp{r}) \; = \; \text{dom}(Y) \\[6pt]
 \forall \; y \in \text{dom}(Y) : \quad \bp{r}[y] \; = \; \PR{Y = y}
    \end{array}
\right. &
\end{flalign*}
\end{addmargin}
\end{defn}

Since the \textproc{marg} subroutine calls \textproc{genAtoms}, which itself calls \textproc{genAtomsByType}, these two generators shall also be formally specified. Note that, contrarily to the specification given above, the following specifications are not meant to have an interpretation outside of the present proof; these are just technical statements that aim at establishing the proof.

\begin{samepage}
\begin{defn} \label{GA_DEF}
Be a pex $d \triangleq \dpd{Y}$ for some random variable $Y$ and be a binding assertion $\alpha$. The generator \textproc{genAtoms}$(d,\alpha)$ is correct iff it yields $n$ atoms $(v_i,\; p_i,\; \lambda_i)$ such that the following six conditions hold:
\vspace{-12pt}
\begin{addmargin}[4em]{0em}
\begin{flalign*}
& \left\{
    \begin{array}{ll}
 \textbf{GA1}. \;\; \textproc{genAtoms}(d,\alpha) \text{ terminates.} \\[6pt]
 \textbf{GA2}. \;\; \forall \; i: \;\; \alpha \wedge \lambda_i \text{ is a binding assertion} \\[6pt]
 \textbf{GA3}. \;\; \big\{ \alpha \wedge \lambda_i \big\} \text{ are mutually exclusive} \\[6pt]
 \textbf{GA4}. \;\; \forall \; i: \;\; \PR{\lambda_i \given \alpha} \; = \; \dfrac{p_i}{\sum_k p_k} \\[10pt]
 \textbf{GA5}. \;\; \forall \; y: \;\; \alpha \; \wedge \; Y = y \quad \Rightarrow \quad  \Or{i}{v_i=y}{\lambda_i} \\[6pt]
 \textbf{GA6}. \;\; \forall \; i: \;\; \alpha \wedge \lambda_i \quad \Rightarrow \quad Y = v_i    
    \end{array}
\right. &
\end{flalign*}
\end{addmargin}
\end{defn}
\end{samepage}

\begin{defn}  \label{GABT_DEF}
Be a pex $d \triangleq \dpd{Y}$ for some random variable $Y$ and be a binding assertion $\alpha$ not binding $Y$. The generator \textproc{genAtomsByType}$(d,\alpha)$ is correct iff it yields $n$ atoms $(v_i,\; p_i,\; \sigma_i)$ such that the following five conditions hold:
\vspace{-12pt}
\begin{addmargin}[4em]{0em}
\begin{flalign*}
& \left\{
    \begin{array}{ll}
 \textbf{GABT1}. \;\; \textproc{genAtomsByType}(d,\alpha) \text{ terminates.} \\[6pt]
 \textbf{GABT2}. \;\; \forall \; i: \;\; \alpha \wedge \sigma_i \wedge Y = v_i \text{ is a binding assertion} \\[6pt]
 \textbf{GABT3}. \;\; \big\{ \alpha \wedge \sigma_i \wedge Y = v_i \big\} \text{ are mutually exclusive} \\[6pt]
 \textbf{GABT4}. \;\; \forall \; i: \;\; \PR{\sigma_i \wedge Y = v_i \given \alpha} \; = \; \dfrac{p_i}{\sum_k p_k} \\[10pt]
 \textbf{GABT5}. \;\; \forall \; y: \;\; \alpha \; \wedge \; Y = y \quad \Rightarrow \quad  \Or{i}{v_i=y}{\sigma_i} 
    \end{array}
\right. &
\end{flalign*}
\end{addmargin}
\end{defn}

In the previous definitions, the index $i$ is meant to cover the range $[1,n]$ or the empty range if $n = 0$; in GA5 and GABT5, $y$ is meant for any value, whether belonging to dom$(Y)$ or not. These definitions cover the case $n = 0$, where no atom is yielded. In such case, the conditions GA2, GA3, GA4, GA6, GABT2, GABT3 and GABT4 are trivially true;
in such case also, the conditions GA5 and GABT5 have a specific consequence, which shall be detailed in proposition \ref{prop:MA}. Note that conditions GA4 and GABT4, which involve conditional probability on binding assertion $\alpha$, are well defined in any case since $\alpha$ is non-contradictory by definition.

\subsection{Proof}

\noindent
\textbf{Outline of the proof}. Proving the correctness of the Statues algorithm consists in proving that the entry-point subroutine \textproc{marg} is correct (proposition \ref{prop:MA}). This in turn requires proving that \textproc{genAtoms} generator is correct (proposition \ref{prop:GA}) and that \textproc{genAtomsByType} generator is correct (proposition \ref{prop:GABT}); since these generators are mutually recursive, this last proof is done by induction. To balance the complexity, all these proofs use subsidiary propositions, which could be considered as lemmas: propositions \ref{prop:COND} and \ref{prop:MUTEX}, as very general statements, and propositions \ref{prop:GABT_L0} to \ref{prop:GA_TO_GABT_TA}, as statements specific to the afore-mentioned generators. The order of presentation follows a "bottom-up" approach, which is consistent with the mathematical practice. However, for having an overall view of the proof, it may be useful for the reader to go backward, starting from proposition \ref{prop:MA} and following a "top-down" approach.
\\\\
We first state hereafter two general propositions: the first one is a generalization of the conditional probability formula, the second one establishes the conservation of mutual exclusiveness when adding conjunctions.

\begin{prop}  \label{prop:COND}
For any propositions $\alpha,\; \lambda,\; \mu$ such that $\alpha \wedge \lambda$ is not contradictory,
\begin{equation} \PR{\lambda \wedge \mu \given \alpha} \; = \; \PR{\lambda \given \alpha} \PR{\mu \given \alpha \wedge \lambda}\end{equation}
\end{prop}
\begin{proof}
Be the propositions $\alpha,\; \lambda,\; \mu$ such that $\alpha \wedge \lambda$ is not contradictory. Using the formula of conditional probability, we can derive
\begin{align}
\PR{\lambda \wedge \mu \given \alpha} \; & = \; \frac{\PR{\lambda \wedge \mu \wedge \alpha}}{\PR{\alpha}} \; = \; \frac{\PR{\alpha \wedge \lambda} \PR{\mu \given \alpha \wedge \lambda}}{\PR{\alpha}} \nonumber \\
& = \; \frac{\PR{\lambda \wedge \alpha}}{\PR{\alpha}} \PR{\mu \given \alpha \wedge \lambda} \; = \; \PR{\lambda \given \alpha} \PR{\mu \given \alpha \wedge \lambda}
\end{align}
\end{proof}

\begin{prop} \label{prop:MUTEX}
For any set of $n$ mutually exclusive propositions $ \{ \lambda_i \} $ and for any set of $n$ propositions $ \{ \mu_i \} $, the set of conjunctions $ \{ \lambda_i \wedge \mu_i \} $ are mutually exclusive.
\end{prop}
\begin{proof}
Be a set of $n$ mutually exclusive propositions $ \lambda_i $ and be a set of $n$ propositions $ \mu_i $. Be the indexes $i$ and $j$. Using definition \ref{DEF_MUTEX}, we can derive
\begin{align}
(\lambda_i \wedge \mu_i) \wedge (\lambda_j \wedge \mu_j) \quad  & \Rightarrow \quad (\lambda_i \wedge \lambda_j) \wedge (\mu_i \wedge \mu_j) \nonumber \\
 & \Rightarrow \quad i = j
\end{align}
Hence, $ \{ \lambda_i \wedge \mu_i \} $ are mutually exclusive.
\end{proof}

We are now equipped to establish several propositions that eventually prove that \textproc{genAtomsByType} and \textproc{genAtoms} are correct with regard to their specifications given in \ref{GABT_DEF} and \ref{GA_DEF}.

\begin{prop} \label{prop:GABT_TO_GA}
For any pex $d$ and for any binding assertion $\alpha$, if $\alpha$ binds $d$ or if \textproc{genAtomsByType}$(d,\alpha)$ is correct, then \textproc{genAtoms}$(d,\alpha)$ is correct.
\end{prop}
\begin{proof}
Be a pex $d \triangleq \dpd{Y}$ for some random variable $Y$ and be a binding assertion $\alpha$. According to algorithm \ref{algo:statues2b}, there are two cases to examine, depending whether $\alpha$ binds $Y$ or not.
\begin{case}
If $\alpha$ binds $Y$, that is \; $\alpha \Rightarrow Y = v$ \; for some value $v$, then \textproc{genAtoms}$(d,\alpha)$ yields the sole atom $(v,\; 1,\; \true)$ (see line 3 of algorithm \ref{algo:statues2b}). As required by the specification \ref{GA_DEF}, there are six statements to verify: GA1, GA2, GA3, GA4, GA5 and GA6.
\\[4pt]
-- \textbf{GA1}: \textproc{genAtoms}$(d,\alpha)$ terminates since no loop is executed.
\\[4pt]
-- \textbf{GA2}: $\alpha \wedge \true $ is a binding assertion, by assumption.
\\[4pt]
-- \textbf{GA3}: $\big\{ \alpha \wedge \true \big\}$ is trivially mutually exclusive since it is a singleton.
\\[4pt]
-- \textbf{GA4}: $\PR{\true \given \alpha} \; = \; \dfrac{1}{1} $ is trivially verified.
\\[4pt]
-- \textbf{GA5}: Be a value $y$.
The implication to prove is
\begin{equation} \label{GA5_1}
\alpha \; \wedge \; Y = y \quad \Rightarrow \quad  \Or{i}{v=y}{\true}
\end{equation}
There are two cases to verify:
\begin{addmargin}[1em]{0em}
- if $y = v$, then (\ref{GA5_1}) becomes  $ \alpha \wedge Y = v \; \Rightarrow \; \true$; \\
- if $y \neq v$, then (\ref{GA5_1}) becomes $ \alpha \wedge Y = y \; \Rightarrow \; \false $.\\
\end{addmargin}
\vspace{-10pt}
In both cases, the statement (\ref{GA5_1}) is verified due to the condition $\alpha \Rightarrow Y = v$ stated in the present case.
\\[4pt]
-- \textbf{GA6}: The implication $\alpha \wedge \true \; \Rightarrow \; Y = v$ is trivially verified since it is the condition stated in the present case.
\end{case}
\begin{case}
If $\alpha$ does not bind $Y$, then, according to lines 5-7 of algorithm \ref{algo:statues2b}, \textproc{genAtomsByType}$(d,\alpha)$ yields $(v_i,\; p_i,\; \sigma_i)$. Since we assume here that this generator is correct, this set of atoms verify the statements GABT1, GABT2, GABT3, GABT4 and GABT5. Then, \textproc{genAtoms}$(d,\alpha)$ yields $(v_i,\; p_i,\; \sigma_i \wedge Y = v_i)$. Again, there are six statements to verify: GA1, GA2, GA3, GA4, GA5 and GA6.
\\[4pt]
-- \textbf{GA1}: \textproc{genAtoms}$(d,\alpha)$ terminates since GABT1 ensures that the executed loop terminates.
\\[4pt]
-- \textbf{GA2}, \textbf{GA3}, \textbf{GA4}: these three statements are trivially verified since, in the present case, these are equivalent respectively to GABT2, GABT3, GABT4, which are verified by assumption.
\\[4pt]
-- \textbf{GA5}: Be a value $y$. Using GABT5, we can derive
\begin{align}
\alpha \; \wedge \; Y = y \quad & \Rightarrow \quad \big( \Or{i}{v_i=y}{\sigma_i} \; \big)  \; \wedge \;  Y = y \nonumber \\
& \Rightarrow \quad \Or{i}{v_i=y}{\big( \sigma_i \;  \wedge \;  Y = v_i\big)}
\end{align}
-- \textbf{GA6}: Be an index $i$. The implication $\alpha \wedge \sigma_i \wedge Y = v_i \; \Rightarrow \; Y = v_i$ is trivially verified.
\end{case}
\end{proof}

The following five propositions establish the correctness of ${\textproc{genAtomsByType}}$  (algorithm \ref{algo:statues3b}) for each of the five types of pex. For the four derived pex types, the correctness of \textproc{genAtoms} is assumed. These propositions shall then be used in the proof by induction stating the unconditional correctness of ${\textproc{genAtomsByType}}$ (proposition \ref{prop:GABT}). For tuple, conditional and table pex, the algorithm \ref{algo:statues3b} executes two embedded loops: the iterations on the outer loop shall be numbered by the index $i$ and the iterations on the inner loop shall be numbered by the index pair $(i,j)$.

\begin{prop} \label{prop:GABT_L0}
For any elementary pex $d$ and for any binding assertion $\alpha$ not binding $d$, ${\textproc{genAtomsByType}(d,\alpha)}$ is correct.
\end{prop}
\begin{proof}
Be an elementary pex $d \triangleq \dpd{X}$ for some elementary random variable $X$. According to lines 6-8 of algorithm \ref{algo:statues3b}, the ${\textproc{genAtomsByType}(d,\alpha)}$ generator yields $(v_i,\; p_i,\; \true)$ for the $n$ elements of the pmf, with $n \ge 1$. As required by the specification \ref{GABT_DEF}, there are five statements to verify: GABT1, GABT2, GABT3, GABT4 and GABT5.
\\[4pt]
-- \textbf{GABT1}: The ${\textproc{genAtomsByType}(d,\alpha)}$ generator terminates after having yielded the $n$ elements of its pmf, which is finite by definition.
\\[4pt]
-- \textbf{GABT2}: Be an index $i$. $\alpha \wedge \true \wedge X = v_i $ is a binding assertion since, by assumption, $\alpha$ is a binding assertion not binding $X$ and since $\PR{X = v_i} > 0$ by definition of a pmf.
\\[4pt]
-- \textbf{GABT3}: $\big\{ \alpha \wedge \true \wedge X = v_i \big\} \text{ are mutually exclusive}$ by application of proposition \ref{prop:MUTEX}, given that the $v_i$ values are distinct by definition of a pmf.
\\[4pt]
-- \textbf{GABT4}: Be an index $i$. Since $\alpha$ is not binding $X$, we have
\begin{equation} \PR{X = v_i \given \alpha} \; = \;  \PR{X=v_i} \end{equation}
By definition of a pmf, $p_i \triangleq \PR{X=v_i} $ and $\sum_k p_k = 1$. So the expected equality
\begin{equation} \PR{\true \wedge X = v_i \given \alpha} \; = \; \dfrac{p_i}{\sum_k p_k} \end{equation}
is verified.
\\[4pt]
-- \textbf{GABT5}: Be a value $x$. The implication to verify is
\begin{equation} \label{GABT5_1}
\alpha \; \wedge \; X = x \quad \Rightarrow \quad  \Or{i}{v_i=x}{\true} 
\end{equation}
There are two cases to distinguish, whether $x$ belongs to $\text{dom}(X)$ or not.
\begin{addmargin}[1em]{0em}
- if $x \in \text{dom}(X)$, then (\ref{GABT5_1}) becomes $ \alpha \wedge X = x \; \Rightarrow \; \true$, which is trivially verified;\\
- if $x \not \in \text{dom}(X)$, then the disjonction has no member and is $\false$ by definition; then (\ref{GABT5_1}) becomes $ \alpha \wedge X = x \; \Rightarrow \; \false$, which is verified since $X \neq  x$.
\end{addmargin}
\end{proof}

Because the generators are mutually recursive, their correctness is proved by induction. This proof technique requires associating a natural number to each pex, in order to be able to express the base case and the inductive step. For this purpose, we introduce the concept of level of a pex:

\begin{defn}
The \emph{level} of a given pex $d$ is defined as follows:
\begin{equation}
\text{level}(d) \; \triangleq \;
\left\{
  \begin{array}{ll}
    0            & \text{if $d$ is elementary} \\
    \text{max}\;\{\;\text{level}(c) \given c $ is child of $ d \; \} + 1 & \text{if $d$ is derived}
  \end{array}
\right.
\end{equation}
\end{defn}

The level of a given pex can be interpreted as the length of the longest path in its associated DAG. For instance, the pex represented on figure \ref{fig:pex-1} has a level 4 and the pex represented on figure \ref{fig:pex-2} has a level 7. This concept of level 
is needed to prove the following propositions.

\begin{prop} \label{prop:GA_TO_GABT_FU}
For any functional pex $d$ and for any binding assertion $\alpha$ not binding $d$, if \textproc{genAtoms}$(x,\alpha)$ is correct for any pex $x$ such that ${ \text{level}(x) \le \text{level}(d) - 1 }$, then ${\textproc{genAtomsByType}(d,\alpha)}$ is correct.
\end{prop}
\begin{proof}
Be a functional pex $d \triangleq \dpd{Y} \triangleq \efunc{f}{\dpd{X}}$ for some function $f$ and some random variable $X$. Be a binding assertion $\alpha$ not binding $Y$. According to lines 11-13 of the algorithm \ref{algo:statues3b}, ${\textproc{genAtoms}(\dpd{X},\alpha)}$ yields $(v_i,\; p_i,\; \lambda_i)$ at iteration $i$. Since $\text {level}(d) \triangleq \text{level}(\dpd{X})+1$, we have
$\text{level}(\dpd{X}) \le \text{level}(d) - 1$, hence we can assume here that this \textproc{genAtoms} generator is correct and that it verifies the statements GA1, GA2, GA3, GA4, GA5 and GA6. Then, ${\textproc{genAtomsByType}(d,\alpha)}$ yields $(f(v_i),\; p_i,\; \lambda_i)$. As required by the specification \ref{GABT_DEF}, there are five statements to verify: GABT1, GABT2, GABT3, GABT4 and GABT5.
\\[4pt]
-- \textbf{GABT1}: \textproc{genAtomsByType}$(d,\alpha)$ terminates since GA1 ensures that the executed loop terminates and since $f(v_i)$ terminates in any case ($f$ is a true function by definition of functional pex).
\\[4pt]
-- \textbf{GABT2}: Be an index $i$. Using GA6, we can derive
\begin{align}
\alpha \wedge \lambda_i \quad & \Rightarrow \quad X = v_i \nonumber \\
                              & \Rightarrow \quad Y \triangleq f(X) = f(v_i) \label{f_1}
\end{align}
Since $ \alpha \wedge \lambda_i $ is a binding assertion by GA2, $ \alpha \wedge \lambda_i \wedge Y = f(v_i) $ is also a binding assertion, given that (\ref{f_1}) ensures that there is no contradiction. 
\\[4pt]
-- \textbf{GABT3}: $\big\{ \alpha \wedge \lambda_i \wedge Y = f(v_i) \big\}$ are mutually exclusive by application of proposition \ref{prop:MUTEX}, given that $ \big\{\alpha \wedge \lambda_i \big\} $ are mutually exclusive by GA3.
\\[4pt]
-- \textbf{GABT4}: Be an index $i$. Since GA2 ensures that $ \alpha \wedge \lambda_i$ is non contradictory, we can use proposition \ref{prop:COND}:
\begin{equation} \label{f_2}
\PR{\lambda_i \wedge Y = f(v_i) \given \alpha} \; = \; \PR{\lambda_i \given \alpha} \PR{Y = f(v_i) \given \alpha \wedge \lambda_i}
\end{equation}
The first factor can be replaced by a fraction using GA4 equality; according to (\ref{f_1}), the second factor is equal to 1. Hence, we get the expected equality:
\begin{equation} \PR{\lambda_i \wedge Y = f(v_i) \given \alpha} \; = \; \dfrac{p_i}{\sum_k p_k} \end{equation}
-- \textbf{GABT5}: Be a value $y$. Using the definition of $Y$ and GA5 relations, we can derive
\begin{align}
& \alpha \; \wedge \; Y = y & \Rightarrow & \quad \alpha \; \wedge \; f(X) = y \nonumber\\
 \Rightarrow \quad & \alpha \; \wedge \; \Or{i}{f(v_i)=y}{X = v_i} & \Rightarrow & \quad \Or{i}{f(v_i)=y}{\big( \alpha \; \wedge \; X = v_i \big)} \nonumber \\
 \Rightarrow \quad & \Or{i}{f(v_i)=y}{\big( \Or{j}{v_j=v_i}{ \lambda_j} \big)} & \Rightarrow & \quad \Or{i}{f(v_i)=y}{\lambda_i}
\end{align}
\end{proof}

\begin{prop} \label{prop:GA_TO_GABT_TU}
For any tuple pex $d$ and for any binding assertion $\alpha$ not binding $d$, if \textproc{genAtoms}$(x,\alpha)$ is correct for any pex $x$ such that ${ \text{level}(x) \le \text{level}(d) - 1 }$, then ${\textproc{genAtomsByType}(d,\alpha)}$ is correct.
\end{prop}
\begin{proof}
\sloppy Be a tuple random variable $Y \triangleq \tuple{H \centerdot T}$ for some random variables $H$ and $T$, where $T$ is either a tuple random variable or the empty tuple $\tuple{}$. Be the tuple pex $d \triangleq \dpd{Y} = \dpd{H} \otimes \dpd{T}$ and be a binding assertion $\alpha$ not binding $Y$. According to lines 16-20 of algorithm \ref{algo:statues3b}, ${\textproc{genAtoms}(\dpd{H},\alpha)}$ yields $(v_i,\; p_i,\; \lambda_i)$ at iteration $i$ and ${\textproc{genAtoms}(\dpd{T},\alpha \wedge \lambda_i)}$ yields $(s_{ij},\; q_{ij},\; \mu_{ij})$ at iteration $(i,j)$. Since ${ \text{level}(d) \triangleq \text{max(level}(\dpd{H}),\text{level}(\dpd{T})) + 1 }$, we have
$\text{level}(\dpd{H}) \le \text{level}(d) - 1$ and $\text{level}(\dpd{T}) \le \text{level}(d) - 1$, hence we can assume here that these \textproc{genAtoms} generators are correct and that they verify the statements GA1, GA2, GA3, GA4, GA5 and GA6. Then, the ${\textproc{genAtomsByType}(d,\alpha)}$ yields ${ (\tuple{v_i \centerdot s_{ij}},\; p_i.q_{ij},\; \lambda_i \wedge \mu_{ij}) }$ at iteration $(i,j) $. As required by the specification \ref{GABT_DEF}, there are five statements to verify: GABT1, GABT2, GABT3, GABT4 and GABT5.
\\[4pt]
-- \textbf{GABT1}: \textproc{genAtomsByType}$(d,\alpha)$ terminates since the GA1 conditions ensure that the executed loops terminate.
\\[4pt]
-- \textbf{GABT2}: Be the indexes $i,j$. Using GA6, we can derive
\begin{align}
\alpha \wedge \lambda_i \wedge \mu_{ij} \quad & \Rightarrow \quad H = v_i \; \wedge \; T = s_{ij} \nonumber \\
 & \Rightarrow \quad Y \triangleq \tuple{H \centerdot T} = \tuple{v_i \centerdot s_{ij}} \label{t_1}
\end{align}
Since $ \alpha \wedge \lambda_i \wedge \mu_{ij} $ is a binding assertion by GA2, $ \alpha \wedge \lambda_i \wedge \mu_{ij} \wedge Y = \tuple{v_i \centerdot s_{ij}} $ is also a binding assertion, given that (\ref{t_1}) ensures that there is no contradiction.
\\[4pt]
-- \textbf{GABT3}: $\big\{ \alpha \wedge \lambda_i \wedge \mu_{ij} \wedge Y = \tuple{v_i \centerdot s_{ij}} \big\}$ are mutually exclusive by application of proposition \ref{prop:MUTEX}, given that $ \big\{\alpha \wedge \lambda_i \wedge \mu_{ij}\big\} $ are mutually exclusive by GA3.
\\[4pt]
-- \textbf{GABT4}: Be the indexes $i,j$. Since GA2 ensures that $ \alpha \wedge \lambda_i \wedge \mu_{ij}$ is non contradictory, we can use proposition \ref{prop:COND} in chain:
\begin{align}
& \PR{\lambda_i \wedge \mu_{ij} \wedge Y = \tuple{v_i \centerdot s_{ij}} \given \alpha} \nonumber \\
& \; = \; \PR{\lambda_i \given \alpha} \PR{\mu_{ij} \given \alpha \wedge \lambda_i} \PR{Y = \tuple{v_i \centerdot s_{ij}} \given \alpha \wedge \lambda_i \wedge \mu_{ij} }
\end{align}
The first two factors can be replaced by fractions using GA4 equalities. According to (\ref{t_1}), the third factor is equal to 1. Hence, we get the expected equality:
\begin{align}
\PR{\lambda_i \wedge \mu_{ij} \wedge Y = \tuple{v_i \centerdot s_{ij}} \given \alpha} & \; = \; \frac {p_i} {\sum _k p_k} \frac {q_{ij}} {\sum _{k,h} q_{kh}} \nonumber \\
 & \; = \; \frac {p_i q_{ij}} {\sum _{k,h} p_k q_{kh}}
\end{align}
-- \textbf{GABT5}: Be a value $y$. There are two cases to distinguish, depending whether $y$ is a tuple or not.
\begin{case}
If $y$ is not a tuple, then the expected implication
\begin{equation} \alpha \; \wedge \; Y = y \quad \Rightarrow \quad \Or{i,j}{\tuple{v_i \centerdot s_{ij}}=y}{\big( \lambda_i \wedge \mu_{ij} \big)} \end{equation}
is verified since the disjonction has no member, hence both sides of the implication are false.
\end{case}
\begin{case}
If $y$ is a tuple, then $y \triangleq \tuple{h \centerdot t}$ for some value $h$ and some tuple $t$. Using GA5 relations, we can derive
\begin{align}
& \alpha \; \wedge \; Y = y & \Rightarrow & \quad \alpha \; \wedge \tuple{H \centerdot T} = \tuple{h \centerdot t} \nonumber \\
 \Rightarrow \quad  & \alpha \; \wedge \; H = h \; \wedge \; T = t & \Rightarrow & \quad \alpha \wedge \big( \Or{i}{v_i=h}{\lambda_i} \big) \; \wedge \; T = t & \nonumber \\
 \Rightarrow \quad & \Or{i}{v_i=h}{\big( \alpha \wedge \lambda_i \wedge T = t \big)} & \Rightarrow & \quad \Or{i}{v_i=h}{\big( \lambda_i \wedge \Or{j}{s_{ij}=t}{\mu_{ij}} \big)}  \nonumber\\ 
 \Rightarrow \quad & \Or{i,j}{v_i=h \; \wedge \; s_{ij}=t}{\big( \lambda_i \wedge \mu_{ij} \big)} & \Rightarrow & \quad \Or{i,j}{\tuple{v_i \centerdot s_{ij}}=\tuple{h \centerdot t}}{\big( \lambda_i \wedge \mu_{ij} \big)} \nonumber \\
 \Rightarrow \quad & \Or{i,j}{\tuple{v_i \centerdot s_{ij}}=y}{\big( \lambda_i \wedge \mu_{ij} \big)} && 
\end{align}
\end{case}
\end{proof}

\begin{prop}  \label{prop:GA_TO_GABT_CO}
For any conditional pex $d$ and for any binding assertion $\alpha$ not binding $d$, if \textproc{genAtoms}$(x,\alpha)$ is correct for any pex $x$ such that ${ \text{level}(x) \le \text{level}(d) - 1 }$, then \textproc{genAtomsByType}$(d,\alpha)$ is correct.
\end{prop}
\begin{proof}
\sloppy Be a conditional random variable $Y \triangleq X \given E$ for some random variables $X$ and $E$, where $E$ is boolean. Be the conditional pex ${ d \triangleq \dpd{Y} = \egiven{\dpd{X}}{\dpd{E}} }$ and be a binding assertion $\alpha$ not binding $Y$. According to the lines 23-29 of algorithm \ref{GABT_DEF}, \textproc{genAtoms}$(\dpd{E},\alpha)$ yields $(v_i,\; p_i,\; \lambda_i)$ at iteration $i$ and, if $v_i$ is $\true$, ${\textproc{genAtoms}(\dpd{X},\alpha \wedge \lambda_i)}$ yields $(s_{ij},\; q_{ij},\; \mu_{ij})$ at iteration $(i,j)$. Since $\text{level}(d) \triangleq \text{max(level}(\dpd{X}),\text{level}(\dpd{E})) + 1$, we have
${\text{level}(\dpd{X}) \le \text{level}(d) - 1}$ and ${\text{level}(\dpd{E}) \le \text{level}(d) - 1}$, hence we can assume here that these \textproc{genAtoms} generators are correct and that they verify the statements GA1, GA2, GA3, GA4, GA5 and GA6. Then, ${\textproc{genAtomsByType}(d,\alpha)}$ yields $(s_{ij},\; p_i.q_{ij},\; \lambda_i \wedge \mu_{ij})$ at iteration $(i,j)$ for the index $i$ such that $v_i$ is $\true$. As required by the specification \ref{GABT_DEF}, there are five statements to verify: GABT1, GABT2, GABT3, GABT4 and GABT5.
\\[4pt]
-- \textbf{GABT1}: \textproc{genAtomsByType}$(d,\alpha)$ terminates since the GA1 conditions ensure that the executed loops terminate.
\\[4pt]
-- \textbf{GABT2}: Be the indexes $i,j$ such that $v_i$ is $\true$. Using GA6, we can derive
\begin{align}
\alpha \wedge \lambda_i \wedge \mu_{ij} \quad & \Rightarrow \quad  X = s_{ij} \nonumber \\
& \Rightarrow \quad  Y \triangleq X = s_{ij} \label{c_1}
\end{align}
Since $ \alpha \wedge \lambda_i \wedge \mu_{ij} $ is a binding assertion by GA2, $ \alpha \wedge \lambda_i \wedge \mu_{ij} \wedge Y = s_{ij} $ is also a binding assertion, given that (\ref{c_1}) ensures that there is no contradiction.
\\[4pt]
-- \textbf{GABT3}: $ \big\{ \alpha \wedge \lambda_i \wedge \mu_{ij} \wedge Y = s_{ij} \big\} $ where $v_i$ is $\true$ are mutually exclusive  by application of proposition \ref{prop:MUTEX}, given that $ \big\{\alpha \wedge \lambda_i \wedge \mu_{ij}\big\} $ are mutually exclusive by GA3.
\\[4pt]
-- \textbf{GABT4}: Be the indexes $i, j$ such that $v_i$ is true. Since GA2 ensures that $ \alpha \wedge \lambda_i \wedge \mu_{ij}$ is non contradictory, we can use proposition \ref{prop:COND} in chain:
\begin{align}
& \PR{\lambda_i \wedge \mu_{ij} \wedge Y = s_{ij} \given \alpha} \nonumber \\
& \; = \; \PR{\lambda_i \given \alpha} \PR{\mu_{ij} \given \alpha \wedge \lambda_i} \PR{Y = s_{ij} \given \alpha \wedge \lambda_i \wedge \mu_{ij} }
\end{align}
The first two factors can be replaced by fractions using GA4 relations; according to (\ref{c_1}), the third factor is equal to 1. Hence, we get the expected equality
\begin{align}
\PR{\lambda_i \wedge \mu_{ij} \wedge Y = s_{ij} \given \alpha} & \; = \; \frac {p_i} {\sum _k p_k} \frac {q_{ij}} {\sum _{k,h} q_{kh}} \nonumber \\
& \; = \; \frac {p_i q_{ij}} {\sum _{k,h} p_k q_{kh}}
\end{align}
-- \textbf{GABT5}: Be a value $y$. Using the definition of $Y$ and GA5 relations, we can derive
\begin{align}
& \alpha \; \wedge \; Y = y & \Rightarrow & \quad \alpha \; \wedge \; E = \true \; \wedge \; X = y \nonumber \\
 \Rightarrow \quad & \alpha \; \wedge \; \big( \Or{i}{v_i=\true}{\lambda_i} \big) \; \wedge \; X = y  
 & \Rightarrow & \quad \Or{i}{v_i=\true}{\big( \alpha \wedge \lambda_i \; \wedge \; X = y \big)} \nonumber \\  
 \Rightarrow \quad & \Or{i}{v_i=\true}{\big( \lambda_i \wedge \Or{j}{s_{ij}=y}{\mu_{ij}} \big)} & \Rightarrow & \quad \Or{i,j}{v_i=\true \; \wedge \; s_{ij}=y}{\big( \lambda_i \wedge \mu_{ij} \big)}
\end{align}
\end{proof}

\begin{prop}  \label{prop:GA_TO_GABT_TA}
For any table pex $d$ and any binding assertion $\alpha$ not binding $d$, if \textproc{genAtoms}$(x,\alpha)$ is correct for any pex $x$ such that ${ \text{level}(x) \le \text{level}(d) - 1 }$, then \textproc{genAtomsByType}$(d,\alpha)$ is correct.
\end{prop}
\begin{proof}
\sloppy Be a table random variable ${Y \triangleq C \; \unrhd \; T}$ for some random variable $C$ and for table ${ T \triangleq \bigl\{ {c_1:X_1, \;\ldots\;, c_n:X_n} \bigl\} }$ associating each value $c_i$ of dom$(C)$ to some random variable $X_i$. According to the definition of a table RV, we can write also ${Y \triangleq T[C]}$, which is more convenient for the rest of the proof. Be the table pex ${ d \triangleq \dpd{Y} = \dpd{C} \circledcirc g }$, where ${ g \triangleq \bigl\{ {c_1:\dpd{X_1}, \;\ldots\;, c_n:\dpd{X_n}} \bigl\} }$. Be a binding assertion $\alpha$ not binding $Y$. According to lines 32-36 of algorithm \ref{algo:statues3b}, ${\textproc{genAtoms}(\dpd{C},\alpha)}$ yields $(v_i,\; p_i,\; \lambda_i)$ at iteration $i$ and ${\textproc{genAtoms}(g[v_i],\alpha \wedge \lambda_i)}$ yields $(s_{ij},\; q_{ij},\; \mu_{ij})$ at iteration $(i,j)$. Since $\text{level}(d) \triangleq \text{max(level}(\dpd{C}),\text{level}(\dpd{X_i})) + 1$, we have
${\text{level}(\dpd{C}) \le \text{level}(d) - 1}$ and ${\text{level}(\dpd{X_i}) \le \text{level}(d) - 1}$, hence we can assume here that these \textproc{genAtoms} generators are correct and that they verify the statements GA1, GA2, GA3, GA4, GA5 and GA6. Then, ${\textproc{genAtomsByType}(d,\alpha)}$ yields $(s_{ij},\; p_i.q_{ij},\; \lambda_i \wedge \mu_{ij})$ at iteration $(i,j)$ for the index $i$. As required by the specification \ref{GABT_DEF}, there are five statements to verify: GABT1, GABT2, GABT3, GABT4 and GABT5.
\\[4pt]
-- \textbf{GABT1}: \textproc{genAtomsByType}$(d,\alpha)$ terminates since the GA1 conditions ensure that the executed loops terminate.
\\[4pt]
-- \textbf{GABT2}: Be the indexes $i,j$.  Be the indexes $i,j$. Using GA6, we can derive
\begin{align}
\alpha \wedge \lambda_i \wedge \mu_{ij} \quad & \Rightarrow \quad C = v_i \; \wedge \; T[v_i] = s_{ij} \nonumber \\
 & \Rightarrow \quad Y \triangleq T[C] = s_{ij} \label{ta_1}
\end{align}
Since $ \alpha \wedge \lambda_i \wedge \mu_{ij} $ is a binding assertion by GA2, $ \alpha \wedge \lambda_i \wedge \mu_{ij} \wedge Y = s_{ij} $ is also a binding assertion, given that (\ref{ta_1}) ensures that there is no contradiction.
\\[4pt]
-- \textbf{GABT3}: $ \big\{ \alpha \wedge \lambda_i \wedge \mu_{ij} \wedge Y = s_{ij} \big\} $ are mutually exclusive by application of proposition \ref{prop:MUTEX}, given that $ \big\{\alpha \wedge \lambda_i \wedge \mu_{ij}\big\} $ are mutually exclusive by GA3.
\\[4pt]
-- \textbf{GABT4}: Be the indexes $i, j$. Since GA2 ensures that $ \alpha \wedge \lambda_i \wedge \mu_{ij}$ is non contradictory, we can use proposition \ref{prop:COND} in chain:
\begin{align}
& \PR{\lambda_i \wedge \mu_{ij} \wedge Y = s_{ij} \given \alpha} \nonumber \\
& \; = \; \PR{\lambda_i \given \alpha} \PR{\mu_{ij} \given \alpha \wedge \lambda_i} \PR{Y = s_{ij} \given \alpha \wedge \lambda_i \wedge \mu_{ij} }
\end{align}
The first two factors can be replaced by fractions using GA4 relations. According to (\ref{ta_1}), the third factor is equal to 1. Hence, we get the expected equality
\begin{align}
\PR{\lambda_i \wedge \mu_{ij} \wedge Y = s_{ij} \given \alpha} & \; = \; \frac {p_i} {\sum _k p_k} \frac {q_{ij}} {\sum _{k,h} q_{kh}} \nonumber \\
& \; = \; \frac {p_i q_{ij}} {\sum _{k,h} p_k q_{kh}}
\end{align}
-- \textbf{GABT5}: Be a value $y$. Using the definition of $Y$ and GA5 relations, we can derive
\begin{align}
& \alpha \; \wedge \; Y = y & \Rightarrow & \quad \alpha \; \wedge \; T[C] = y \nonumber \\
 \Rightarrow \quad & \alpha \; \wedge \; \bigvee_k \big( C = v_k \; \wedge \; T[v_k] = y  \big)
& \Rightarrow & \quad \bigvee_k \big(\alpha \; \wedge \; C = v_k \; \wedge \; T[v_k] = y  \big) \nonumber \\
 \Rightarrow \quad & \bigvee_k \big(\alpha \; \wedge \; \big( \Or{i}{v_i=v_k}{\lambda_i} \big) \; \wedge \; T[v_k] = y  \big) & \Rightarrow & \quad \bigvee_i \big(\alpha \wedge \lambda_i \; \wedge \; T[v_i] = y \big) \nonumber \\
 \Rightarrow \quad & \bigvee_i \big(
 \lambda_i \wedge \Or{j}{s_{ij} = y}{\mu_{ij}} \big) & \Rightarrow & \quad \Or{i,j}{s_{ij}=y}{\big( \lambda_i \wedge \mu_{ij} \big)}
\end{align}
\end{proof}

The previous propositions can now be used to prove the unconditional correctness of the \textproc{genAtomsByType} and \textproc{genAtoms} generators.

\begin{prop}  \label{prop:GABT}
For any pex $d$ and any binding assertion $\alpha$ not binding $d$,\\ ${\textproc{genAtomsByType}(d,\alpha)}$ is correct.
\end{prop}
\begin{proof}
\sloppy Be a pex $d$ and a binding assertion $\alpha$. The proof goes by induction, using level$(d)$. If level($d$) = 0, then $d$ is an elementary pex and \textproc{genAtomsByType}$(d,\alpha)$ is correct by proposition \ref{prop:GABT_L0}. Consider now the case where level$(d) \ge 1$. Suppose that ${\textproc{genAtomsByType}(x,\alpha)}$ is correct for any pex $x$ with ${ \text{level}(x) \le \text{level}(d) - 1 }$ (inductive step assumption). Under this assumption, by proposition \ref{prop:GABT_TO_GA}, \textproc{genAtoms}$(x,\alpha)$ is also correct for any pex $x$ with ${ \text{level}(x) \le \text{level}(d) - 1 }$. We have now to prove that \textproc{genAtomsByType}$(d,\alpha)$ is correct. Since level$(d) \ge 1$, $d$ is of one of the following pex type: functional, tuple, conditional or table; the correctness of ${\textproc{genAtomsByType}(d,\alpha)}$ has been proved for any of these type in propositions \ref{prop:GA_TO_GABT_FU}, \ref{prop:GA_TO_GABT_TU}, \ref{prop:GA_TO_GABT_CO} and \ref{prop:GA_TO_GABT_TA}, respectively. This concludes the induction, so \textproc{genAtomsByType}$(d,\alpha)$ is correct whatever the value of $\text{level}(d)$.
\end{proof}

\begin{prop}  \label{prop:GA}
For any pex $d$ and any binding $\alpha$, \textproc{genAtoms}$(d,\alpha)$ is correct.
\end{prop}
\begin{proof}
This is a direct application of propositions \ref{prop:GABT} and \ref{prop:GABT_TO_GA}.
\end{proof}

Now that the correctness of the generator \textproc{genAtoms} has been established according to its specification \ref{GA_DEF}, we are able to conclude the proof of the Statues algorithm, by proving the correctness of the \textproc{marg} subroutine, as specified in \ref{MA_DEF}.

\begin{prop} \label{prop:MA}
For any pex $d$, \textproc{marg}$(d)$ is correct.
\end{prop}
\begin{proof}
Be a pex $d \triangleq \dpd{Y}$ for some random variable $Y$. According to line 3 of algorithm \ref{algo:statues1b}, ${\textproc{genAtoms}(d,\true)}$ yields $(v_i,\; p_i,\; \lambda_i)$ at iteration $i$. By proposition \ref{prop:GA}, since $\true$ is a binding assertion, \textproc{genAtoms}$(d,\true)$ is correct: it yields $n$ atoms $(v_i,\; p_i,\; \lambda_i)$ that verify conditions GA1, GA2, GA3, GA4, GA5 and GA6 (see \ref{GA_DEF}) with $\alpha \triangleq \true$. \textproc{marg}$(d)$ terminates because the sole loop on atoms yielded by \textproc{genAtoms}$(d,\true)$ is ensured to terminate according to GA1. There are two cases to verify, depending whether atoms are yielded or not.
\begin{case}
If no atom is yielded ($n = 0$), then the condition GA5 becomes
\begin{align}
& \forall \; y: \; \true \wedge Y = y \quad \Rightarrow \quad  \false \nonumber \\
\text{or, equivalently: \quad } & \forall \; y: \; Y \neq y
\end{align}
which expresses that dom$(Y)$ is empty. Since $n=0$, the $a$ pmf remains empty: according to algorithm \ref{algo:statues1b} (lines 9-11), the \textproc{marg} function reports an error. This is the correct behavior when dom$(Y)$ is empty, according to definition \ref{MA_DEF}.
\end{case}
\begin{case}
If at least one atom is yielded ($n \ge 1$), then the \textproc{genAtoms} function returns a pmf $\bp{r}$ from the atoms $(v_i, p_i ,\lambda_i)$ yielded by \textproc{genAtoms}. Be a value $y$. We examine two cases, depending whether $y$ belongs or not to dom$(Y)$.
\begin{addmargin}[1em]{0em}
-- If $y \not \in \text{dom}(Y)$, then, according to GA5, no atom is yielded with $v_i = y$; according to the algorithm \ref{algo:statues1b}, ${y \not \in \text{dom}(\bp{r})}$.\\
-- If $y \in \text{dom}(Y)$, then, according to GA5, some atom is yielded with $v_i = y$; 
according to the algorithm \ref{algo:statues1b}, ${y \in \text{dom}(\bp{r})}.$
\end{addmargin}
From these two cases, we have then proved that $\text{dom}(Y) = \text{dom}(\bp{r})$, which is the first condition required for \textproc{marg} to be correct (see \ref{MA_DEF}). 
\\
Let us proof now the second condition, which assumes that $y \in \text{dom}(Y)$. The condensation part of the algorithm \ref{algo:statues1b} (lines 2 to 8) entails
\begin{equation} \bp{r}[y] \; = \; \frac{\sum\limits_{\substack{i\\v_i=y}} p_i}{\;\;\sum\limits_k p_k\;\;} \; = \; \sum\limits_{\substack{i\\v_i=y}} \frac{p_i}{\sum\limits_k p_k} \end{equation}
Using GA4, the $p_i$ in the right-hand part can be replaced by probabilities:
\begin{equation} \bp{r}[y] \; = \; \sum_{\substack{i\\v_i=y}} \PR{\lambda_i} \end{equation}
Using GA3 (mutual exclusiveness) and the third axiom of probability, we can replace the sum of probabilities by a probability of a disjunction:
\begin{equation} \bp{r}[y] \; = \; \PR{ \Or {i}{v_i=y}{\lambda_i} } \end{equation}
Merging GA5 and GA6, we establish the following equivalence:
\begin{equation} \Or{i}{v_i=y}{\lambda_i} \quad \Leftrightarrow \quad Y = y\end{equation}
hence
\begin{equation} \bp{r}[y] \; = \; \PR{Y = y} \end{equation}
which is the second and last condition required for \textproc{marg} to be correct (see \ref{MA_DEF}).
\end{case}
\end{proof}

The proof of the correctness of the \textproc{marg} subroutine in proposition \ref{prop:MA} establishes the correctness of the Statues algorithm.

\end{appendices}

\end{document}